\documentclass[10pt]{article}

\usepackage[accepted]{rlj} 

\usepackage{blindtext}
\usepackage{lastpage}
\usepackage[utf8]{inputenc} 
\usepackage[T1]{fontenc} 
\usepackage{hyperref} 
\usepackage{url} 
\usepackage{booktabs} 
\usepackage{amsfonts} 
\usepackage{nicefrac} 
\usepackage{microtype} 
\usepackage{xcolor} 
\usepackage{graphicx}
\usepackage{subcaption}

\usepackage{amsmath}
\usepackage{amssymb}
\usepackage{mathtools}
\usepackage{amsthm}
\usepackage{thmtools}   
\usepackage{thm-restate}

\theoremstyle{plain}

\newtheorem{lemma}{Lemma}
\newtheorem{corollary}{Corollary}
\theoremstyle{definition}

\theoremstyle{remark}

\newcommand{\R}{\mathbb{R}}
\newcommand{\N}{\mathbb{N}}
\newcommand{\prob}{\mathbb{P}} 
\newcommand{\E}{\mathbb{E}} 

\newcommand{\D}{\mathcal{D}}

\newcommand{\sts}{\mathcal{S}} 
\newcommand{\as}{\mathcal{A}} 
\newcommand{\mdp}{\mathcal{M}} 
\newcommand{\tmax}{t_{\max}} 
\newcommand{\argmax}{\text{argmax} }
\newcommand{\VaR}{{\text{VaR}} }

\newcommand{\EIG}{\text{EIG}}

\newcommand{\pia}{{\pi^{\text{A}}}}

\newcommand{\en}{e_{\text{n}}}

\newcommand{\ec}{e_{\text{c}}}

\newcommand{\ea}{e_{\text{a}}}

\title{PAC Apprenticeship Learning with \\ Bayesian Active Inverse Reinforcement Learning}

\setrunningtitle{PAC Apprenticeship Learning with Bayesian Active IRL}

\author{Ondrej Bajgar\textsuperscript{1}, Dewi S.W. Gould\textsuperscript{2}, Jonathon Liu\textsuperscript{3}, \\ 
Alessandro Abate\textsuperscript{1}, Konstantinos Gatsis\textsuperscript{4}, Michael A. Osborne\textsuperscript{1}}

\emails{ondrej@bajgar.org}

\affiliations{
$^{1}$\textbf{University of Oxford, United Kingdom}\\
$^{2}$\textbf{Alan Turing Institute, United Kingdom}\\
$^{3}$\textbf{Independent}\\
$^{4}$\textbf{University of Southampton, United Kingdom}\\
}

\contribution{

    We formulate two principled information theoretic acquisition functions for active inverse reinforcement learning with Boltzmann rational demonstrations: Reward-EIG and PAC-EIG.
    }
    {
    This gives a more principled alternative to previous, heuristic acquisition functions of \citet{10.1007/978-3-642-04174-7_3}, \citet{brown2018}, and \citet{kweon2023}.
    }

\contribution{
    For RegretEIG, we prove a lower bound on the expected number of steps of active learning needed to reach a probably-approximately-correct (PAC) apprentice policy.
    }
    {
    This a first such proof for active IRL with an expert that is not perfectly rational. \citet{metelli2021,lindner2022} presented results for the, in many respects much simpler, case of perfectly optimal expert, focusing especially on transfer of a learnt reward to new environment dynamics.
    }


\summary{
    As AI systems become increasingly autonomous, reliably aligning their decision-making to human preferences is essential. Inverse reinforcement learning (IRL) offers a promising approach to infer preferences from demonstrations. These preferences can then be used to produce an apprentice policy that performs well on the demonstrated task. However, in domains like autonomous driving or robotics, where errors can have serious consequences, we need not just good average performance but reliable policies with formal guarantees -- yet obtaining sufficient human demonstrations for reliability guarantees can be costly. \emph{Active} IRL addresses this challenge by strategically selecting the most informative scenarios for human demonstration. We introduce PAC-EIG, an information-theoretic acquisition function that directly targets probably-approximately-correct (PAC) guarantees for the learned policy -- providing the first such theoretical guarantee for active IRL with noisy expert demonstrations. Our method maximises information gain about the regret of the apprentice policy, efficiently identifying states requiring further demonstration. We also present Reward-EIG as an alternative when learning the reward itself is the primary objective. Focusing on finite state-action spaces, we prove convergence bounds, illustrate failure modes of prior heuristic methods, and demonstrate our method's advantages experimentally.

}

\begin{document}

\fancyhead[R]{}
\maketitle  

\begin{abstract}
      As AI systems become increasingly autonomous, reliably aligning their decision-making with human preferences is essential. Inverse reinforcement learning (IRL) offers a promising approach to infer preferences from demonstrations. These preferences can then be used to produce an apprentice policy that performs well on the demonstrated task. However, in domains like autonomous driving or robotics, where errors can have serious consequences, we need not just good average performance but reliable policies with formal guarantees -- yet obtaining sufficient human demonstrations for reliability guarantees can be costly. \emph{Active} IRL addresses this challenge by strategically selecting the most informative scenarios for human demonstration. We introduce PAC-EIG, an information-theoretic acquisition function that directly targets probably-approximately-correct (PAC) guarantees for the learned policy -- providing the first such theoretical guarantee for active IRL with noisy expert demonstrations. Our method maximises information gain about the regret of the apprentice policy, efficiently identifying states requiring further demonstration. We also present Reward-EIG as an alternative when learning the reward itself is the primary objective. Focusing on finite state-action spaces, we prove convergence bounds, illustrate failure modes of prior heuristic methods, and demonstrate our method's advantages experimentally.
\end{abstract}
\section{Introduction}
Stuart Russell suggested three principles for the development of beneficial artificial intelligence: its only objective is to realize human preferences, it is initially uncertain about these preferences, and its ultimate source of information about them is human behaviour \citep{zotero-1078}. \emph{Apprenticeship learning} via Bayesian \emph{inverse reinforcement learning} (IRL) can be understood as a possible operationalization of these principles: Bayesian IRL starts with a prior distribution over reward functions representing initial uncertainty about human preferences.
It then combines this prior with \emph{demonstration} data from a human expert acting approximately optimally with respect to the unknown reward, to produce a posterior distribution over rewards. In apprenticeship learning, this posterior over rewards is then used to produce a policy that should perform well with respect to the unknown reward function. 

However, getting human demonstrations requires scarce human time. Also, many risky situations where we would wish AI systems to behave especially reliably may be rare in naturally occurring demonstration data. Bayesian active learning can help with both by giving queries to a human demonstrator that are likely to bring the most useful information about the reward. 

Prior methods for active IRL each suffer from significant limitations: \citet{metelli2021} provide a largely theoretical treatment assuming perfectly optimal expert demonstrations, which not only is a strong assumption, but also hinders identifiability. None of the methods that can address noisy demonstrations provide theoretical guarantees. Furthermore, most methods \citep{10.1007/978-3-642-04174-7_3,brown2018,metelli2021} query the expert for action annotations of particular isolated states. However, in domains such as autonomous driving with a high frequency of actions, it can be much more natural for the human to provide whole trajectories -- say, to drive for a while in a simulator -- than to annotate a large collection of unrelated snapshots. There is one previous paper on \emph{active IRL with full trajectories} \citep{kweon2023} suggesting a heuristic acquisition function whose shortcomings can, however, completely prevent learning, as we will demonstrate. Instead, we propose using the principled tools of Bayesian active learning, formulate two methods that can query for full trajectories, and provide theoretical guarantees for one of them. While in this paper, we work in the setting of finite state and action spaces, the methods are designed to suitably generalize to continuous settings, which we plan in future work.

The article provides the following contributions: 

\begin{enumerate}
    \item We explain and demonstrate failure modes of existing heuristic methods for active IRL when the goal is to produce a well-performing apprentice policy. In particular, most previous methods are limited to querying for only a single state annotation, as opposed to whole trajectories. Furthermore, we show that the only prior method designed for querying whole trajectories can result in repeatedly querying a single uninformative state forever.
    \item We propose PAC-EIG, an acquisition function based on expected information gain (EIG) that directly targets \textit{probably approximately correct} (PAC) guarantees for the apprentice policy -- providing the first such theoretical guarantee for active IRL with imperfect expert demonstrations.
    \item We present Reward-EIG as an alternative when learning the reward itself is the primary objective.
    \item We prove convergence bounds showing the expected number of expert demonstrations needed to achieve PAC guarantees.
    \item We illustrate the performance of our methods in a set of gridworld experiments, demonstrating their effectiveness compared to prior heuristic approaches.
\end{enumerate}

\section{Task formulation}
\label{sec:task}

\begin{figure*}[t]
    \centering
    \begin{subfigure}[b]{0.32\textwidth}
        \centering
        \includegraphics[width=\textwidth]{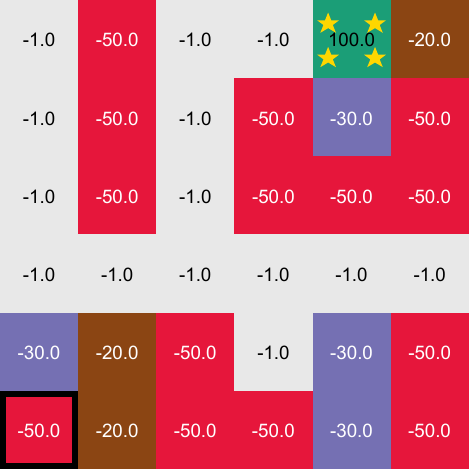}
        \caption{Ground-truth rewards.}
        \label{fig:pedagogical_setup}
    \end{subfigure}
    \hfill
    \begin{subfigure}[b]{0.33\textwidth}
        \centering
        \includegraphics[width=\textwidth]{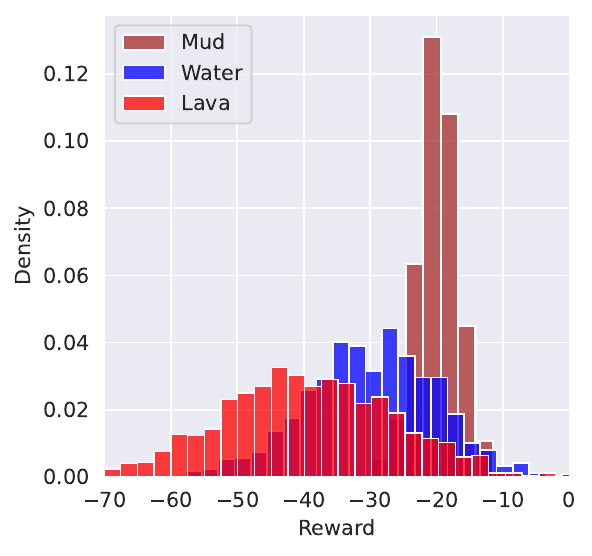}
        \caption{Current belief over rewards.}
        \label{fig:pedagogical_dist}
    \end{subfigure}
    \hfill
    \begin{subfigure}[b]{0.32\textwidth}
        \centering
        \includegraphics[width=\textwidth]{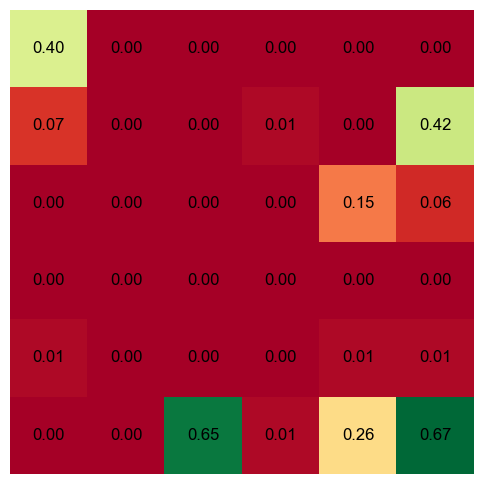}
        \caption{Reward EIG of each initial state.}
        \label{fig:pedagogical_eig}
    \end{subfigure}
    \caption{Illustration of the active IRL task. (a) shows a gridworld and its true rewards. The lower left corner has a "jail" state with negative reward from which an agent cannot leave. 
    The starred green state is the terminal "goal" state with a large positive reward. 
    The brown, blue, and red states are "mud", "water", and "lava" type states respectively, whose rewards are unknown to the IRL agent.
    The IRL agent tries to learn the rewards of these three state types from expert demonstrations.
    (b) shows current distributions over the rewards of the "mud", "water", and "lava" state types respectively, at some particular step of the active learning process. 
    These learned reward distributions are used to calculate an acquisition function (here the reward EIG) of obtaining another expert demonstration starting from each given state, shown in (c). 
    In this case, a demonstration starting in the bottom right state gives the most information about the unknown reward parameters.}
    \label{fig:illustrative_pedagogical}
\end{figure*}

Let $\mdp=\left(\sts, \as, p, r, \gamma, \tmax, \rho\right)$ be a parameterized Markov decision process (MDP), where $\sts$ and $\as$ are finite state and action spaces respectively, $p: \sts \times \as \to \mathcal P(\sts)$ is the transition function where $\mathcal P(\sts)$ is a set of probability 
measures over $\sts$, $r: \sts \times \as \to \mathbb{R}$ is an (expected) reward function,\footnote{\label{footnote:stochastic-rewards}Our formulation permits the reward to be stochastic. However, our expert model~(\ref{eq:boltzmann-rat}) depends on the rewards only via the optimal Q-function, which in turn depends only on the expected reward. Thus, the demonstrations can only ever give us information about the expectation. Throughout the paper, the learnt reward function can be interpreted either as modelling a deterministic reward, or an expectation of a stochastic reward.}
$ \gamma\in (0,1) $ is a discount rate, $\tmax\in\N\cup\{\infty\}$ is the time horizon, and $\rho$ is the initial state distribution. We assume the learner has full knowledge of the MDP except for the reward.

We assume we are initially uncertain about the reward $r$, and our initial knowledge is captured by a prior distribution $p(r)$ over rewards, which is a distribution over $\R^{|\sts| |\as|}$ -- a space of vectors representing the reward associated with each state-action pair (there may be fewer reward parameters than $|\sts| |\as|$, but that can be seen as a special case). We also have access to an expert that, given an initial state $s_0$ of the MDP, can produce a trajectory $\tau_i = \left((s^i_0,a^i_0),\dots,(s^i_{n_i},a^i_{n_i})\right)$, where $s_0^i\sim \rho$, $s_{t+1}\sim p(\cdot|s_t,a_t)$, and
 \begin{equation}
    \label{eq:boltzmann-rat} 
    \pi^E (a_t | s_t) = \frac{\exp(\beta Q^*(s_t,a_t))}{\sum_{a'\in\as} \exp(\beta Q^*(s_t,a'))} \; ,
\end{equation} 
which is called a \emph{Boltzmann-rational} policy, given the optimal Q-function $Q^*$ and a coefficient $\beta$ expressing how close to optimal the expert behaviour is (where $\beta=0$ corresponds to fully random behaviour and $\beta \to + \infty$ would yield the optimal policy). We assume $\beta$ is known as is usual in related IRL literature \citep{ramachandran2007,chan2021,kweon2023,bajgar2024}. We will also denote by $\pi^{\text{E}}_r$ the hypothetical expert policy that would correspond to a reward $r$.

 The task of \emph{Bayesian active inverse reinforcement learning} is to sequentially query the expert to provide demonstrations from initial states $\xi_1,\dots,\xi_N\in\sts$ to gain maximum information about the unknown reward.\footnote{Since the queries $\xi_i$ in this paper are limited to the choice of the initial state $s_0$, $\xi$ and $s_0$ are used somewhat interchangeably throughout the paper.} We start with a (possibly empty) set of expert trajectories $\D_0$ and then, at each step of active learning, we choose an initial state $\xi_i$ for the MDP, from which we get the corresponding expert trajectory $\tau_i$. We then update our demonstration dataset to $\D_i=\D_{i-1}\cup\{\tau_i\}$, and the distribution over rewards to $p(r|\D_i)$, which we again use to select the most informative initial state $\xi_{i+1}$ in the next step. We repeat until we exhaust our limited demonstration budget $N$.

This can be done with one of two possible objectives in mind.

The first, which we call the \emph{reward-learning objective}, is relevant when our primary interest is in the reward itself, e.g. when using IRL to understand the motivations of mice in a maze \citep{ashwood2022} or the preferences of drivers \citep{huang2022}. In the active setting, we operationalize this objective as trying to minimize the entropy of the posterior distribution over rewards, once all expert demonstrations have been observed. This is equivalent to maximizing the log likelihood of the true parameter value in expectation, or to maximizing the mutual information between the demonstrations and the reward. Figure~\ref{fig:illustrative_pedagogical} illustrates Active IRL with this objective.

The second possible objective, which we term the \emph{apprenticeship-learning objective}, uses the final posterior $p(r|\D_N)$ to produce an \emph{apprentice policy} that should perform well in the MDP. One option may be to optimize the expected return of the apprentice policy, i.e. solve for $\argmax_\pi \E_{r|\D_N} [\E_\tau [\sum_{s_t,a_t\in\tau}\gamma^t r(s_t,a_t)]]$,
where $\tau$ is a trajectory with $s_0\sim\rho$, $s_{t+1}\sim p(\cdot|s_t,a_t)$ and $a_t = \pi(s_t)$. The argmax can be resolved by solving the forward planning problem for finding the optimal policy for the expected reward with respect to the learner's current posterior over rewards (e.g. using generalized policy iteration \citep{sutton2018}). Going forward, we generally assume a deterministic apprentice policy, i.e. the class of policies we search over is the set of mappings $\pi:\sts\to\as$, though stochastic policies could easily be accommodated as well. The apprentice policy will thus be distinct from the stochastic (Boltzmann rational) expert policy and with enough knowledge can have higher expected return, since the expert gives non-zero probability to sub-optimal actions. 

However, maximising expected return may not be sufficient in safety-critical domains. We may instead require a \emph{reliable} apprentice policy that performs well with high probability -- formally, one that is $\epsilon$-$\delta$-\emph{probably approximately correct} (PAC). This means finding an apprentice policy $\pi^{\text{A}}$ such that the probability (with respect to the reward posterior) of the expected return (with respect to initial state and transition distributions) being at least $G^*-\epsilon$ is at least $1-\delta$, where $G^*$ is the expected return of the optimal policy.

These objectives are often closely connected -- learning about the reward function enables improving the apprentice policy. However, especially in the active setting, they can come apart -- for instance, once we know an action $a$ leads to lower return than $a'$ in a particular state, we may no longer need to gather further information about rewards in these states for the apprenticeship learning objective as we already know to choose the better action, while the reward-learning objective may motivate further queries to reduce the reward uncertainty. 

Stemming from a common inspiration in Bayesian active learning, we will present an acquisition function tailored to each of these objectives.

\paragraph{Notation} By $V_r^\pi$ we denote the state-value function of policy $\pi$ with respect to reward $r$. $V_r^*$ is then the value function of the optimal policy with respect to $r$. A lack of subscript, as in $V^*$, indicates (optimal) value with respect to the true reward (since the true reward is not known by the learner, this generally needs to be treated as random variable). By $G_r(\tau)$ we denote the return of trajectory $\tau$ with respect to $r$. By $R^\pi_r(s_0)$ we denote the regret of policy $\pi$ starting from state $s_0$, i.e. $R^\pi_r(s_0):=V^*_r(s_0)-V^\pi_r(s_0)$, and $R^\pi_r:=\E_{s_0\sim\rho}R^\pi_r(s_0)$. We also call \emph{immediate regret} the quantity $R^*_{\pi,r}(s)=V^*(s)-Q^*(s,\pi(s))$ and also denote by $R^*_{\pi,r}(s,a)=\max\{0,Q^*(s,a)-Q^*(s,\pi(s))\}$ the immediate regret relative to action $a$ in state $s$. You can also find a table with the notation used in this paper in Appendix~\ref{sup:notation}.
\section{Related work}
\label{sec:prior-work}

IRL was first introduced by \citet{russell1998}, preceded by the closely related problem of \emph{inverse optimal control} formulated by \citet{kalman1964}. See \citet{arora2021} and \citet{adams2022} for recent reviews of the already extensive literature on IRL. In our work we build upon the Bayesian formulation of the problem introduced by \citet{ramachandran2007}.

We will now summarize prior work on \textit{active} IRL in particular. We first describe a number of methods that query for single state annotations (which can be cast into our framework from Section~\ref{sec:task} as trajectories of length one), and then describe the one previous method which queries for whole trajectories. Lastly, we review a few other works for setups not directly comparable to ours.

\subsection{Active learning with single action annotations}

The concept of active IRL was first introduced by \citet{10.1007/978-3-642-04174-7_3}. The authors propose an acquisition function equal to the entropy of the posterior predictive distribution about the Boltzmann expert policy, i.e. they query a state maximizing $\alpha^{\text{Lopes}}_n(s) = H(\Pi_s|\D_n)$
where $\Pi_s$ is the vector of expert action probabilities in state $s$ (according to the posterior predictive distribution).

An issue with this approach is that $H(\Pi_s | \D_n)$ does not take into account the effect of improved knowledge on the apprentice policy. For example, we may know the optimal action in a particular state, but with high uncertainty about the exact action probabilities, while another state may have uncertainty about the optimal action, but lower entropy about exact probabilities of actions. Then, $\alpha^{\text{Lopes}}_n(s)$
would prioritize the latter, which may be suboptimal from the apprenticeship learning perspective. See Appendix \ref{sup:failure-modes} for a full example.

\citet{brown2018} query the expert by maximizing the $\delta$-value-at-risk of the policy loss (i.e. regret) of the current apprentice policy starting from the given initial state, computed as
\begin{equation}
    \alpha^{\text{Brown}}_n(s) = \VaR_\delta \left( V^{\pi^*}(s) - V^{\pia}(s) | \D_n \right) \,.
\end{equation}
This is a risk-aware approach: the states with a high risk of the apprentice action being much worse than the expert's action are queried. A limitation of this approach is that regret attributed to some initial state $s$ may be due to a choice made further along the trajectory where an expert query would be more informative as shown in Appendix \ref{sup:failure-modes}.


\subsection{Active learning with full trajectories}

\citet{kweon2023} query full trajectories with a starting state $s_0$ chosen to maximize 
\begin{equation}
    \alpha_n^{\text{Kweon}}(s_0) = \mathbb{E}_{\tau\sim \hat{\pi}_E^{\D_n}}
    \Bigl[
    \sum_{s_t\in\tau} {\tilde\alpha}_n(s_t) | s_0
    \Bigr] \,,
\end{equation}
where
\begin{equation*}
   {\tilde\alpha}_n(s) := H(\hat{\pi}_E^{\D_n}(a|s)) := \sum_a -\hat{\pi}_E^{\D_n}(a|s) \log \hat{\pi}_E^{\D_n}(a|s),
\end{equation*}
is the entropy of $\hat{\pi}^{\D_n}_E$, the posterior predictive distribution over the expert actions at state $s$, estimated from demonstration data $\D_n$.

However, note that this action entropy can remain high even in states where we have perfect knowledge, but multiple actions are equally good, so the Boltzmann rational policy chooses them with equal probabilities, resulting in high action entropy. However, querying for extra demonstrations in such states will bring no useful knowledge. In fact, this can result in learning getting completely stuck, sometimes right at the beginning, preventing \textit{any} learning from taking place. This is the case in the jail environment in Figure~\ref{fig:illustrative_pedagogical}, and we show this in Section~\ref{sec:experiment}.

\subsection{Other settings}

Instead of querying at arbitrary states, \citet{losey2018} and \citet{lindner2022} synthesize a policy that explores the environment to produce a trajectory which subsequently gets annotated by the expert. We instead let the expert produce the trajectory. \citet{buening2024} query full trajectories in the context of IRL, where the active component arises in the choice of a transition function from a set of transition functions at each step. \citet{buning2022} also query full trajectories in a different context involving two cooperating autonomous agents. 
In \citet{sadigh2017}, the expert is asked to provide a relative preference between two sample trajectories synthesized by the algorithm. While this generally provides less information per query than our formulation, it is a useful alternative for situations where providing high-quality demonstrations is difficult for humans.

On the side of theoretical sample complexity of (active) IRL, all prior work assumes a perfectly rational expert policy, which is a stronger assumption than our Boltzmann rationality. In particular, seeing each state once is enough to determine the optimal policy. The first lower bound on the complexity of IRL was given by \citet{komanduru} for the case of a $\beta$-separable finite set of candidate rewards. \citet{metelli2021}, \citet{lindner2022}, and \citet{metelli2023} then focus on recovering a feasible reward set in settings where also the transition dynamics are only estimated, and address the problem of the transferability of the learnt reward to environments with different dynamics.

\section{Method}
\label{sec:method}

We propose PAC-EIG, an acquisition function based on expected information gain (EIG) that aims to produce a probably approximately correct (PAC) apprentice policy. Our approach builds on principled Bayesian experimental design  \citep{rainforth2023} to identify initial states for expert demonstrations that will yield information about the immediate regret of the apprentice policy. The intuition is that knowing about the regret in various states allows us to identify high-regret states where the apprentice policy can be improved. We then also show how EIG can be adapted to the reward-learning objective, resulting in the Reward-EIG acquisition function.

\subsection{PAC-EIG: Information Gain for Reliable Policies}

Our goal is to produce an apprentice policy that is probably approximately correct -- that is, with high probability ($1-\delta$), the policy's regret is bounded by $\epsilon$. To achieve this efficiently, we need to identify states where the current apprentice policy might be making poor decisions. To this end, we define the \emph{immediate regret} of an apprentice policy $\pia$ in state $s$ as:
\begin{equation}
R^*_{\pia,r}(s) = V^*_r(s) - Q^*_r(s, \pia(s))
\end{equation}
which captures how much value we lose by following the apprentice policy in state $s$ compared to optimal behaviour. This can be decomposed per action as $$R^*_{\pia,r}(s,a) = \max\{0, Q^*_r(s,a) - Q^*_r(s,\pia(s))\},$$ representing the regret relative to choosing a particular alternative action $a$. This regret is unknown to us, so we need to treat it as a random variable. We propose to use the expected information gain about this immediate regret as a theoretical acquisition function that helps us find a PAC policy.

For practical computation, we discretize the immediate regret into a ternary variable $E_{s,a}$ tracking state- and action-wise correctness as follows:
\begin{equation}
    \label{eq:discretized-regret}
E_{s,a}^{\pia} = \begin{cases}
    \text{``correct''} & \text{if } R^*_{\pia,r}(s,a) = 0 \\
    \text{``approximately correct''} & \text{if } 0 < R^*_{\pia,r}(s,a) < \epsilon (1-\gamma)\\
    \text{``not correct''} & \text{if } R^*_{\pia,r}(s,a) \geq \epsilon (1-\gamma)
\end{cases}
\end{equation}
Our acquisition function then maximizes the expected information gain about these discretized regret values:
\begin{equation}
    \alpha^{\text{PAC-EIG}}_n(s_0) := I(\tau; E^{\pia} | s_0, \mathcal{D}_n)
\end{equation}
where $E^{\pia} = (E^{\pia}_{s,a})_{s \in \mathcal{S}, a \in \mathcal{A}}$ represents the discretized regret across all state-action pairs, and $\tau$ is the expert trajectory starting from $s_0$. Note that if the apprentice policy is approximately correct in all states with probability at least $1-\delta$ in the immediate regret sense, then we also satisfy the PAC criterion globally.

While we are aiming for a $\epsilon$-$\delta$-PAC policy as the final output, for intermediate steps, our algorithm uses an apprentice policy $\pia(s)=\argmax_a \prob[Q^*(s,\pi^*(s)) - Q^*(s,a)=0|\D_n],\;\forall s\in\sts$, i.e. an apprentice policy maximizes the probability of taking the optimal action in each state. The reason for this choice is that it ensures there is probability mass on both of two contrastive options: the policy being optimal in any given state on the one hand, and, on the other, it having an immediate regret of at least $(1-\gamma)\epsilon$ in at least one state as long as the PAC condition is not satisfied. This contrast ensures that we gain information by observing the expert in this state as we detail in the next section.

This acquisition function acquires information about immediate regret \emph{in every state}, which allows to eventually learn a policy that does not lose more than $(1-\gamma)\epsilon$ of expected return in any given state. This can be useful to learn a uniformly reliable policy that is thus robust e.g. with respect to the choice of initial state. However, immediate regret in different states does not matter equally if we care about overall expected regret. The overall expected regret can be decomposed as $\sum_s\nu^\pia(s)R^*_\pia(s)$ where $\nu^\pia(s)=\sum_t\gamma^t\prob[S_t=s]$ is the discounted expected occupancy of $s$ under the policy $\pia$. This allows us to replace the immediate regret in each state $s$ relative to action $a$ with its weighted version $\tilde{R}^*_\pia(s,a) = \nu(s) R^*_\pia(s,a)$, and obtain an analogous weighted acquisition function $\alpha^{\nu\text{PAC-EIG}}$ (in particular keeping the same fixed thresholds for the discretization), which is suitable especially for larger state spaces when expected occupancy is concentrated in only a smaller subset of the space. We will concentrate subsequent treatment on PAC-EIG for simplicity, but most points can be suitably extended also to $\nu$PAC-EIG.


\subsection{Computing PAC-EIG}

To compute PAC-EIG in practice, we leverage our Bayesian IRL posterior over Q-values. Given $\mathcal{Q}_n$, a set of $M$ samples from $p(Q^*|\mathcal{D}_n)$, we can:
\begin{enumerate}
    \item For each Q-value sample, compute the discretized regret values $E^{\pia}$. Note that multiple Q-value samples may map to the same discretized configuration $E^{\pia}$, so there are at most $M_E \leq M$ distinct values for the samples of $E^{\pia}$.
    \item Given a Q-value sample $Q_i^*$, sample expert trajectories $\tau$ starting from $s_0$ using the Boltzmann policy corresponding to $Q_i^*$.
    \item Estimate the expected information gain using the standard Monte Carlo estimator of EIG as 
    \begin{equation}
        \alpha^{\text{PAC-EIG}}_n(s_0) \approx \frac{1}{M} \sum_{i=1}^M \left[ \log p(\tau^{(i)}|E^{\pia,(i)},s_0) - \log p(\tau^{(i)}|s_0) \right]
    \end{equation}
    where the trajectory probability given $E^{\pia}$ can be computed as $p(\tau|E^{\pia},s_0) = \prod_{(s_t,a_t) \in \tau} p(a_t|s_t,E^\pia)$ (omitting the transition probabilities since they would cancel out in the log-ratio). To compute $p(a_t|s_t,E^{\pia})$, we average the expert action probabilities over all Q-value samples that map to the same discretized configuration: $p(a_t|s_t,E^{\pia}) = \frac{1}{|\mathcal{Q}_{E^\pia}|} \sum_{Q^* \in \mathcal{Q}_{E^\pia}} p(a_t|s_t,Q^*)$, where $\mathcal{Q}_{E^\pia}$ denotes the set of Q-value samples corresponding to the discretized regret configuration $E^{\pia}$ and the action probability the Boltzmann rational expert policy \ref{eq:boltzmann-rat}.
\end{enumerate}


\subsection{Reward EIG: When Learning the Reward is the Goal}

While our primary focus above has been on producing reliable apprentice policies, in some applications the reward function itself is of intrinsic interest -- for instance, when using IRL to understand animal behaviour \citep{ashwood2022} or human preferences \citep{huang2022}. For these cases, we can still use the EIG framework, but instead maximize the expected information gain about the reward:
\begin{equation}
    \alpha^{\text{Reward-EIG}}_n(s_0) := I(\tau;r|s_0,\mathcal{D}_n)
\end{equation}
where $\tau$ is treated as a random variable representing the expert's trajectory that would be produced starting from $s_0$.

This acquisition function aims to reduce posterior uncertainty about the reward parameters, which may query different states than PAC-EIG. For example, it might seek to precisely estimate reward values in states that the apprentice already knows to avoid, whereas PAC-EIG would consider such queries unnecessary.

The reward EIG can be computed as:
\begin{equation}
    \alpha^{\text{Reward-EIG}}_n(s_0) = \mathbb{E}_{r|\mathcal{D}_n} \left[ \mathbb{E}_{\tau|r,s_0}[\log p(\tau|r,s_0) - \log p(\tau|s_0;\mathcal{D}_n)] \right]
\end{equation}
where the inner expectation is tractable to compute from the Q-values that are usually obtained as a byproduct of a Bayesian IRL algorithm.

\section{Producing a PAC Policy}
\label{sec:pac}
Our PAC-EIG acquisition function is designed to efficiently produce a probably-approximately-correct (PAC) apprentice policy.
We will show that PAC-EIG leads to such a policy by establishing bounds on the expected number of expert demonstrations needed. The analysis proceeds through three key steps (with formal results and proofs in Appendix~\ref{sup:theory}):
\begin{enumerate}
	\item If no apprentice policy satisfies the PAC condition, there must exist a state where there is a significant chance that any apprentice policy makes a significantly suboptimal choice -- specifically:
    \begin{equation*}
        \prob_{r|\D_n} \bigl[ V_r^*(s) - Q^*_r(s, \pia(s)) \geq (1 - \gamma) \epsilon \bigr] \geq \frac{\delta}{|\sts|} \tag*{(Lemma~\ref{lemma:probabilistic-q-val-diff}).}
    \end{equation*}
	\item In such a state, there is both a chance that the apprentice policy $\pia$ is close to optimal and that it is significantly suboptimal (as defined in step 1). Since these two options would result in a sufficiently different expert policies in this state, we can gain a lower-bounded expected amount of information by observing the expert in that state.
	
	\item Since we gain at least this minimum information per query while the PAC condition is unmet, and our initial uncertainty is finite, we must eventually achieve the PAC condition. The number of steps is bounded by the ratio of initial entropy to the minimum information gain per step.
\end{enumerate}

These insights translate into the following two theorems:

\begin{restatable}{theorem}{thmmineig}
\label{thm:min-eig}
For $\epsilon > 0$ and $\delta \in (0,\frac{1}{2}]$, assume that no policy $\pia$ is $(\epsilon,\delta)$-probably-approximately-correct, i.e., $\prob[R^{\pi}_r \geq \epsilon] > \delta,\;\forall\pi$. Then, there exists a state $s\in\sts$ such that observing a new expert demonstration at $s$ has an expected information gain with respect to the variable $E^{\pia}$ (Eq. \ref{eq:discretized-regret}) of at least
\begin{equation}
    \EIG_{\text{min}}(\epsilon, \delta) = \frac{\delta(1-e^{-\beta(1-\gamma)\epsilon})^2}{4|\as|^2(|\as|-1)^3|\sts|}.
\end{equation}
\end{restatable}

Then, we can translate this into the following result on the expected number of steps to reach the PAC criterion:
\begin{restatable}{theorem}{thmnumsteps}
\label{thm:num-steps}
Let $h_{\text{max}}$ be an upper bound on the entropy of the prior distribution over $E$, the PAC-EIG discretized regret values $E^{\pia}$ aggregated across all apprentice policies $\pia$. Then, the expected number of steps needed to reach the PAC condition is upper bounded by
\begin{equation}
    \frac{h_{\text{max}}}{\text{EIG}_{\text{min}}(\epsilon, \delta)} = \frac{4h_{\text{max}}|\as|^2(|\as|-1)^3||\mathcal{S}|}{\delta(1-e^{-\beta(1-\gamma)\epsilon})^2}.
\end{equation}
\end{restatable}

For the ternary discretization used in PAC-EIG, the maximum initial entropy is $\log(3)|\mathcal{S}|(|\mathcal{A}|-1)$ for any given policy and the associated variable $E^{\pia}$. However, since the policy may change during the optimization process, we get a total initial entropy across all possible policies of $h_{\text{max}} \leq \log(3)|\mathcal{S}||\mathcal{A}|(|\mathcal{A}|-1)$ (even though there are $|\as|^{|\sts|}$ possible policies, note that the ternary variable in a given state does not depend on the policy's actions in other states, leaving the number of distinct components of $E$ at $|\mathcal{S}||\mathcal{A}|(|\mathcal{A}|-1)$).

\paragraph{Extension to trajectory queries.} While the theorems establish a lower bound for single-state queries, this naturally extends to a trajectory-based version of our PAC-EIG acquisition function. When querying for a trajectory starting from state $s_0$, the information gained is at least as large as querying any single state visited along the trajectory. In practice, trajectories typically visit multiple informative states, potentially providing substantially more information than the theoretical lower bound suggests. However, an improved theoretical bound would need to build on additional assumptions about the environment and the prior.\footnote{For example, the environment may terminate after a single step, forcing the EIG from a trajectory to that of a single state.}

\section{Experiments}
\label{sec:experiment}
We evaluated the performance of the two proposed acquisition functions in a set of gridworld experiments with respect to both objectives introduced earlier: the reward learning objective, measured by the entropy of the posterior distribution over rewards, and the apprenticeship learning objective, measured by the regret. We also track the posterior distribution over regrets which directly relates to the PAC criterion.

We evaluate across three types of environment:
\begin{enumerate}
    \item \textbf{Structured gridworld}: Features fewer reward parameters than states. It includes a known goal state with a reward of +100, neutral states with a reward of -1, and three obstacle types with unknown negative rewards with a uniform prior between -100 and 0 independently for each obstacle type. This was meant as an illustrative example (Figure~\ref{fig:illustrative_pedagogical}) and a counterexample to the only prior method designed for collecting full trajectories, \emph{action entropy} \citep{kweon2023}, by including a jail state where all actions are equivalent and which always gets selected by this baseline, thus preventing any useful learning.
    \item \textbf{10x10 random gridworld with 2 initial states}: Each state has a random reward drawn from the prior, $\mathcal{N}(0,3)$, with only two possible initial states to test the ability of methods to recognize only relevant parts of the state space. Here we use $\beta=4$ so the expert behaves closely to optimal.
    \item \textbf{8x8 random gridworld with fully uniform initial states}: Each state has a random reward drawn from the prior, $\mathcal{N}(0,3)$, and the initial state distribution is uniform across all states. We use $\beta=2$ so the expert is fairly stochastic.
\end{enumerate}

We evaluate both the setting where each query results in a full expert trajectory, where we compare against the only prior method \citep{kweon2023} as well as random sampling, and the setting where each query results in a single state-action annotation, where we also evaluate against the methods by \citet{10.1007/978-3-642-04174-7_3} and \citet{brown2018}.


Each experiment type was run with 16 different random seeds (i.e. combinations of random reward functions, terminal states (except for the jail environment), and different 2 initial states in the 10x10 environment). The plots display the mean and the standard error across these 16 random instances. 

\subsection{Results}

Figure~\ref{fig:jail_results} shows results on the simple environment from Figure~\ref{fig:illustrative_pedagogical} to illustrate a crucial failure mode of the \textit{action entropy} \citep{kweon2023} acquisition function -- it always queries the jail state and thus fails to learn anything useful, while both reward and regret EIG learn an optimal policy within 10 steps in all 16 instances with similar posterior entropies.

Figure~\ref{fig:10x10_results} shows the results on the 10x10 gridworld with 2 random initial states and single-state annotations, Figure~\ref{fig:8x8_results} shows the results for querying single-state annotations on the 8x8 gridworld with a uniform initial state distribution, and Figure~\ref{fig:8x8_fulltraj_results} results on the 8x8 environment when querying trajectories of maximum length 5. 

In the case of only 2 initial states on the 10x10 gridworld, we can see our regret-focused $\nu$PAC-EIG acquisition function to perform much better in terms of both actual and posterior regret, reaching a zero true regret, as well as 0.1-0.1-PAC apprentice policy, by step 50. ActiveVaR remains competitive with RewardEIG in terms of the reward-learning objective (entropy of the reward posterior), but both fall behind $\nu$PAC-EIG in terms of true and posterior regret. While we do not display plain PAC-EIG for better readability, it performed comparably with RewardEIG and ActiveVaR indicating that it is the occupancy weighing that makes the difference here. It is interesting to observe that here the reward learning and apprenticeship objectives do diverge.

On the 8x8 gridworld with fully uniform initial states, we observe that both our information-theoretic acquisition functions result in lower posterior reward entropy \emph{and} lower regret than prior methods except for ActiveVaR, which seems to roughly match their performance. Interestingly, the reward-based acquisition function and the regret-based ones seem to perform similarly well on both objectives, suggesting that there is a strong correlation between learning about the reward and learning about the apprentice regret in this environment with uniform initial states.

The action entropy acquisition function still stops yielding significant improvements after about step 50 -- it again gets stuck querying states that have high action entropy due to multiple actions being similarly good, even if these states do not yield any more information.

\begin{figure*}[t]
    \begin{subfigure}[b]{0.32\textwidth}
        \centering
        \includegraphics[width=\textwidth]{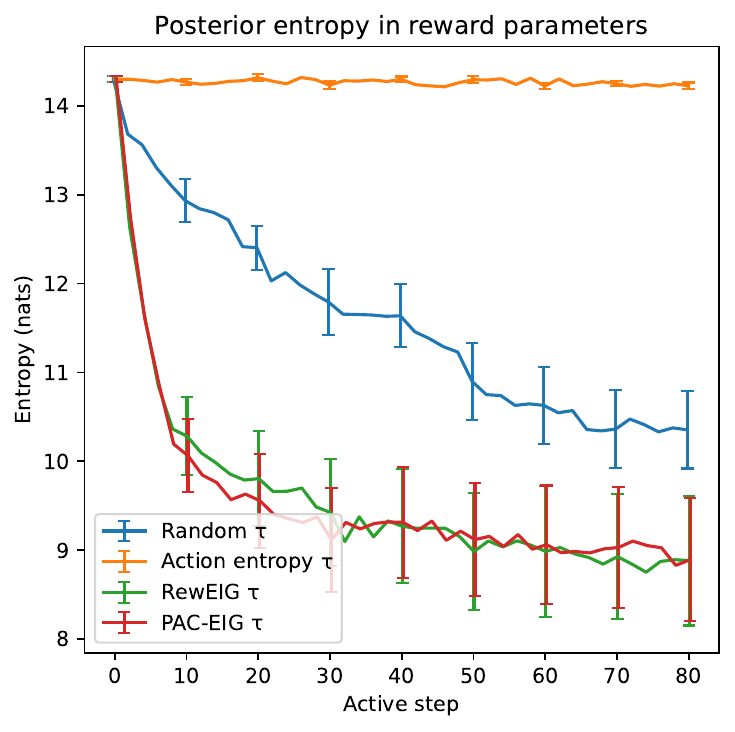}
        \caption{Posterior reward entropy}
        \label{fig:jail_entropy}
    \end{subfigure}
    \begin{subfigure}[b]{0.32\textwidth}
        \centering
        \includegraphics[width=\textwidth]{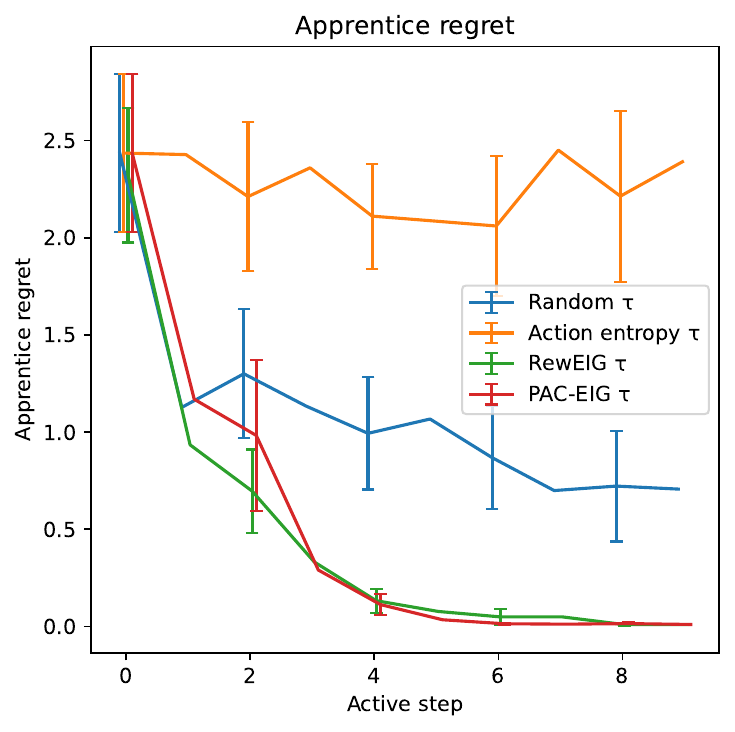}
        \caption{True regret}
        \label{fig:jail_regret}
    \end{subfigure}
    \begin{subfigure}[b]{0.32\textwidth}
        \centering
        \includegraphics[width=\textwidth]{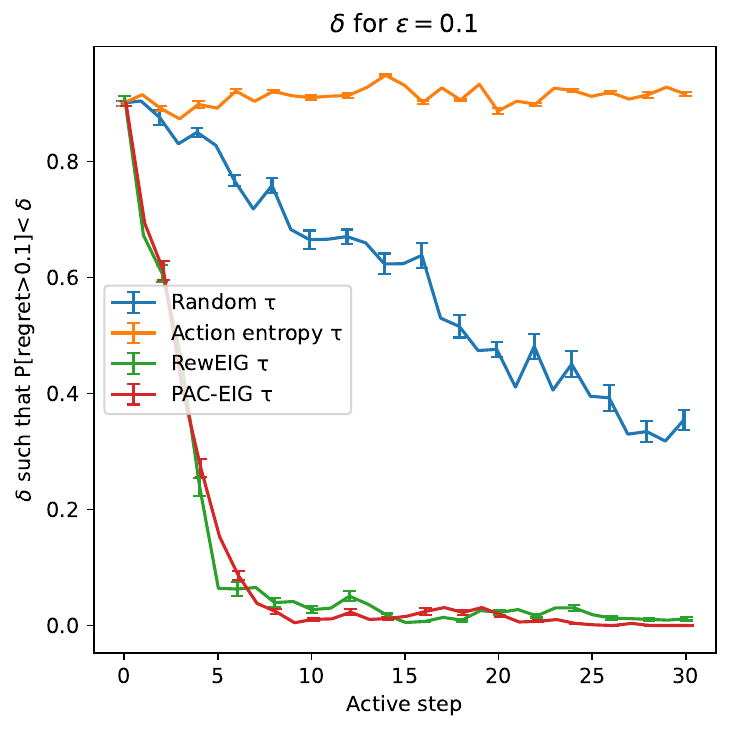}
        \caption{Posterior prob. of regret $>0.1$}
        \label{fig:jail_pac_prob}
    \end{subfigure}
    \caption{Results of the experiments on the environment with 3 cell types and a jail state with full-trajectory demonstrations.}
    \label{fig:jail_results}
\end{figure*}
\begin{figure*}[t]
    \begin{subfigure}[b]{0.32\textwidth}
        \centering
        \includegraphics[width=\textwidth]{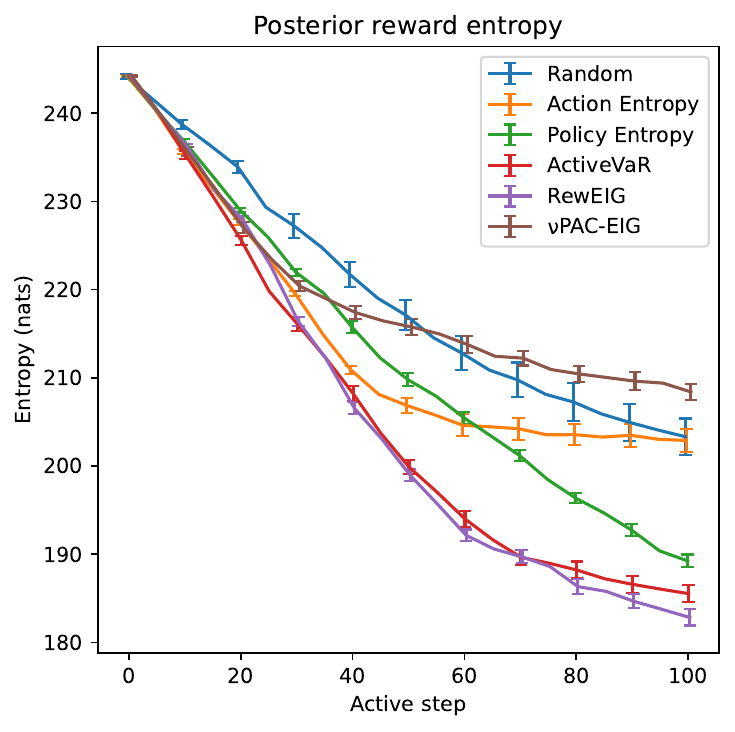}
        \caption{Posterior reward entropy}
        \label{fig:10x10_entropy}
    \end{subfigure}
    \begin{subfigure}[b]{0.32\textwidth}
        \centering
        \includegraphics[width=\textwidth]{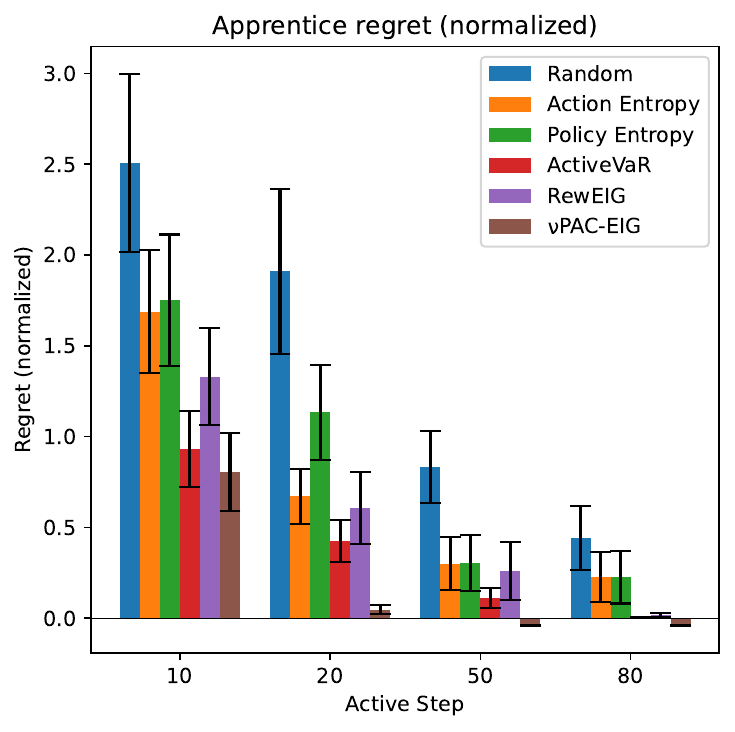}
        \caption{True regret}
        \label{fig:10x10_regret}
    \end{subfigure}
    \begin{subfigure}[b]{0.32\textwidth}
        \centering
        \includegraphics[width=\textwidth]{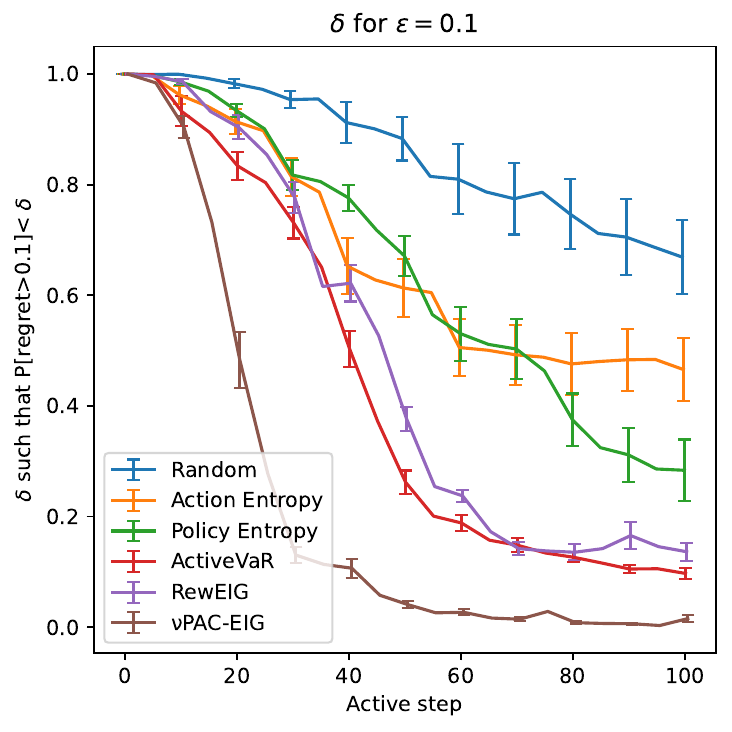}
        \caption{Posterior prob. of regret $>0.1$}
        \label{fig:10x10_pac_prob}
    \end{subfigure}
    \caption{Results of the experiments with single state annotations (i.e. $|\tau|=1$) on the 10x10 fully random gridworld with two initial states. In the barplot (b), results with zero regret are visualized below the horizontal axis to make their presence clearer.}
    \label{fig:10x10_results}
\end{figure*}
\begin{figure*}[t]
    \centering
    \begin{subfigure}[b]{0.32\textwidth}
        \centering
        \includegraphics[width=\textwidth]{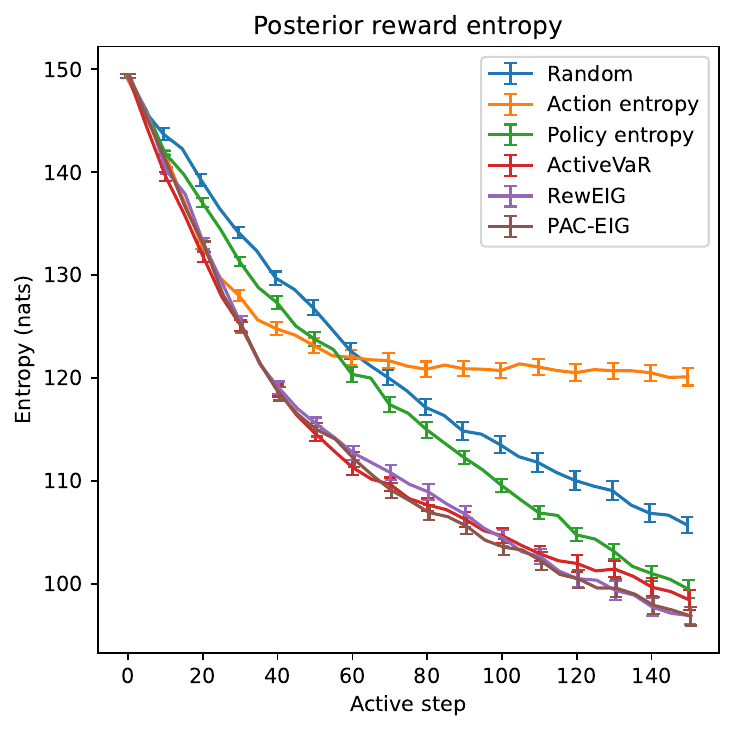}
        \caption{Posterior reward entropy}
        \label{fig:8x8_entropy}
    \end{subfigure}
    \hfill
    \begin{subfigure}[b]{0.32\textwidth}
        \centering
        \includegraphics[width=\textwidth]{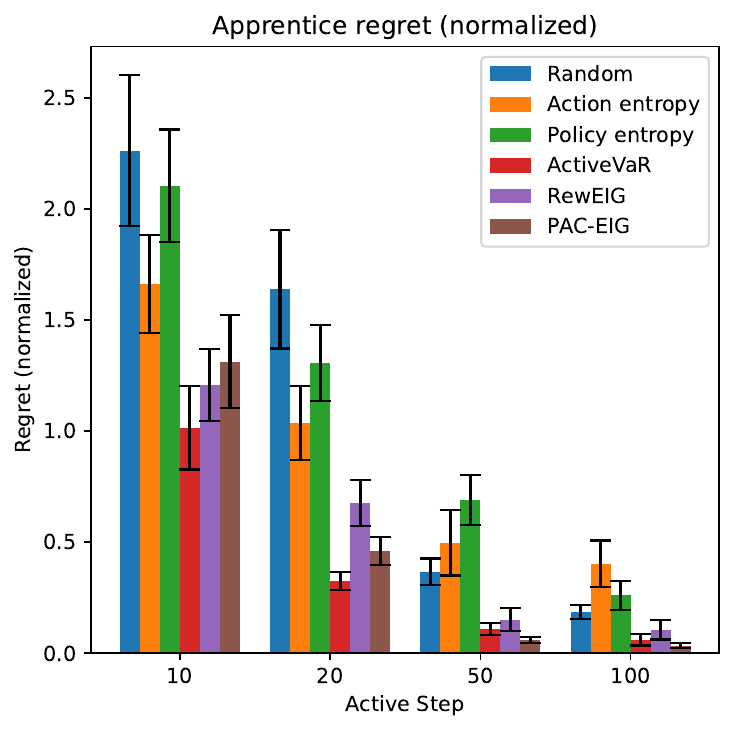}
        \caption{True regret}
        \label{fig:8x8_regret}
    \end{subfigure}
    \hfill
    \begin{subfigure}[b]{0.32\textwidth}
        \centering
        \includegraphics[width=\textwidth]{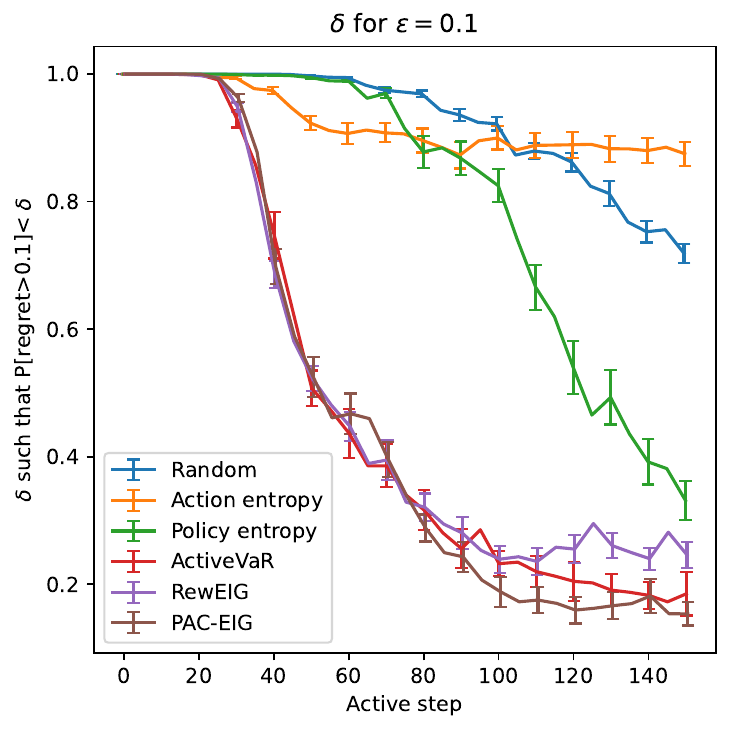}
        \caption{Posterior prob. of regret $>0.1$}
        \label{fig:8x8_pac_prob}
    \end{subfigure}
    \caption{Results of the experiments with single state annotations (i.e. $|\tau|=1$) on the 8x8 fully random gridworld.}
    \label{fig:8x8_results}
\end{figure*}
\begin{figure*}[t]
    \centering
    \begin{subfigure}[b]{0.32\textwidth}
        \centering
        \includegraphics[width=\textwidth]{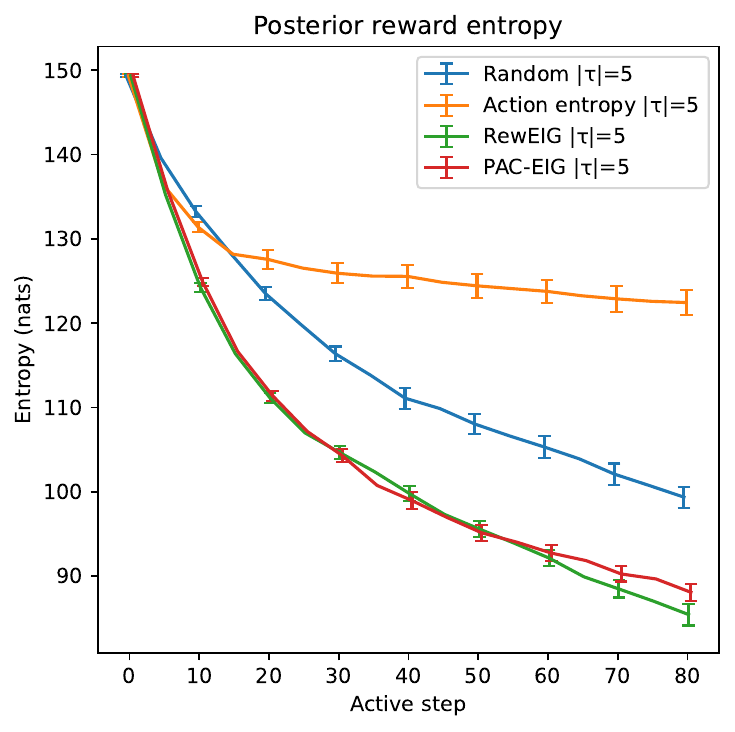}
        \caption{Posterior reward entropy}
        \label{fig:8x8_fulltraj_entropy}
    \end{subfigure}
    \hfill
    \begin{subfigure}[b]{0.32\textwidth}
        \centering
        \includegraphics[width=\textwidth]{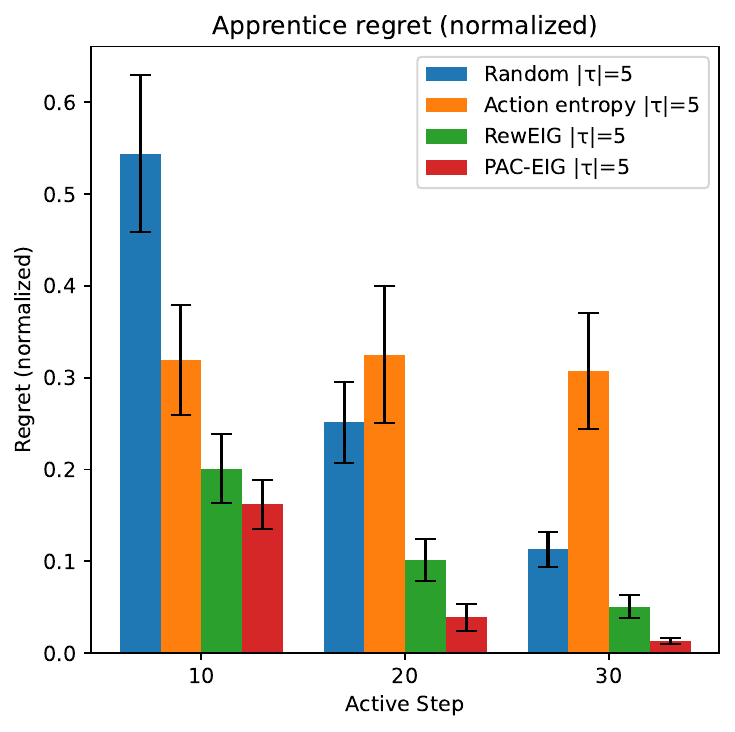}
        \caption{True regret}
        \label{fig:8x8_fulltraj_regret}
    \end{subfigure}
    \hfill
    \begin{subfigure}[b]{0.32\textwidth}
        \centering
        \includegraphics[width=\textwidth]{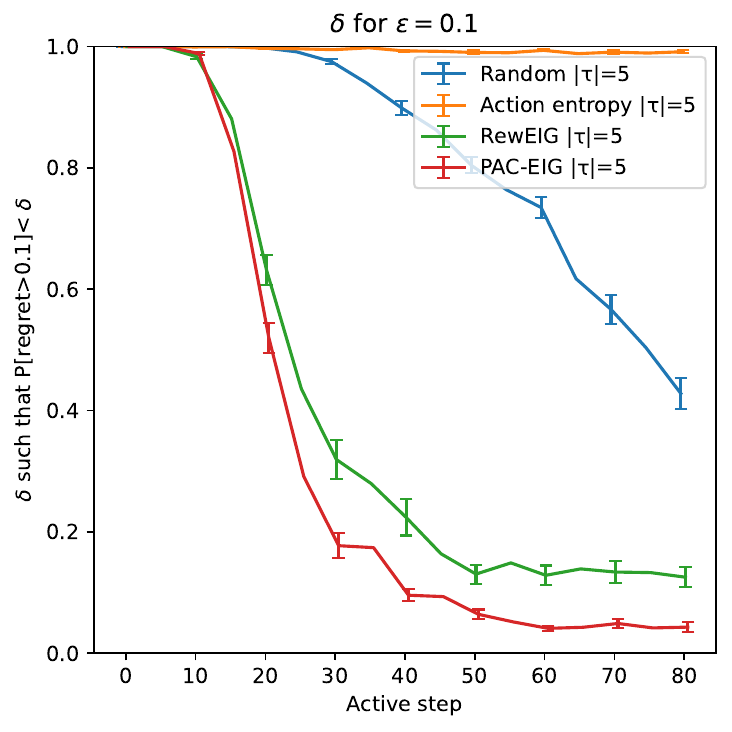}
        \caption{Posterior prob. of regret $>0.1$}
        \label{fig:8x8_fulltraj_pac_prob}
    \end{subfigure}
    \caption{Results of the experiments with expert trajectories of maximum length $|\tau|=5$ on the 8x8 fully random gridworld.}
    \label{fig:8x8_fulltraj_results}
\end{figure*}

\section{Discussion and conclusion}

In this paper, we have proposed new acquisition functions for active IRL, each geared toward one of two possible objectives: learning about an unknown reward function, or producing a well-performing apprentice policy. We have shown that across a set of gridworld experiments, our acquisition functions outperform or at least match prior methods on their respective objectives. Furthermore, our immediate-regret EIG acquisition function is the first acquisition function with a regret bound in our setting. While we have so far tested the methods only in finite state spaces, both of them were constructed to generalize also to continuous spaces, which will be addressed in future work.

\section*{Impact statement}
Through this paper, we hope to contribute to more effective and reliable learning of human preferences and values by AI systems, which aims to improve their alignment and facilitate their beneficial use. 

\bibliography{zotero_bibtex}
\bibliographystyle{rlj}

\newpage

\beginSupplementaryMaterials
\appendix

\section{Failure modes of prior methods}
\label{sup:failure-modes}
For each of the three prior methods for active IRL, we will now present an example of a simple environment where the method makes a clearly suboptimal choice with respect to at least one of the two objectives.

\paragraph{Policy entropy \citep{10.1007/978-3-642-04174-7_3}} As a reminder, the policy entropy acquisition function is $\alpha^{\text{Lopes}}=H(\pi^E)$, i.e. the entropy of the expert policy with respect to the current posterior over rewards (which induces a posterior over expert action probabilities). Consider an environment with two states $s_{0,1}$ each with two actions $a_{1,2}$ as shown in Figure \ref{fig:lopes_env}. We aim to illustrate a scenario where $\alpha^{\text{Lopes}}$ can misallocate budget, from the point of view of the apprenticeship learning objective, by focusing on states where the optimal action is already known, rather than those where crucial information about optimality is missing.

\begin{figure}
    \centering
    \includegraphics[width=0.5\linewidth, trim=0 300 0 300, clip]{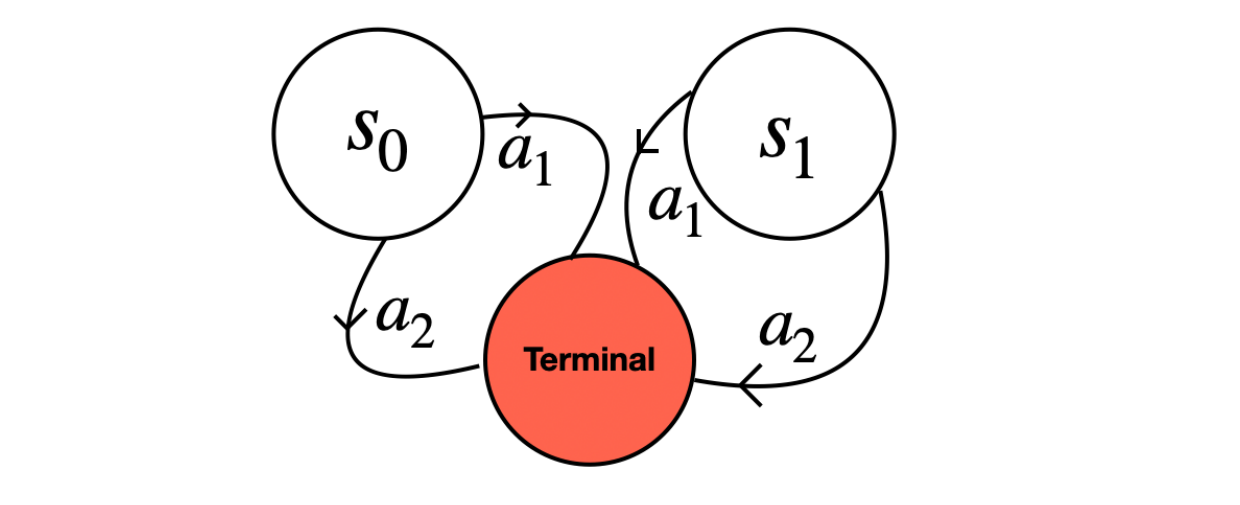}
    \caption{Two-state environment designed to illustrate a failure mode of \citet{10.1007/978-3-642-04174-7_3}}
    \label{fig:lopes_env}
\end{figure}

To demonstrate this effect, we define a \textit{discrete prior distribution over rewards}. This uncertainty in rewards will, in turn, induce a prior over the possible action probabilities for an optimal policy.
Suppose in state $s_1$ we have strong prior knowledge that $a_1$ is the optimal action; however, we are uncertain about the exact reward obtained by taking $a_1$
\begin{equation}
    P(R_{s_1, a_1}=5)=0.5\,, \quad P(R_{s_1,a_1}=7)=0.5 \,,
\end{equation}
and 
\begin{equation}
    P(R_{s_1, a_2}=1) = 1.0 \,.
\end{equation}
An optimal apprentice policy will always choose $a_1$ in state $s_1$. Despite this certainty, the uncertainty in the exact reward for $a_1$ means there is still uncertainty regarding the precise probability an optimal policy would assign to $a_1$, which leads to a high measure of policy uncertainty (as measured by $\alpha^{\text{Lopes}}$).

For state $s_0$, we set priors
\begin{equation}
    P(R_{s_0, a_1} = 2) = 0.1\,, \quad P(R_{s_0, a_1}=3)=0.9 \,, 
\end{equation}
and 
\begin{equation}
    P(R_{s_0, a_2}=2)=0.1 \,, \quad P(R_{s_0,a_2}=3)=0.9 \,.
\end{equation}
such that the optimal action is \textit{uncertain}. In this state, the learner faces true ambiguity about the best action, and it is therefore the state that a good active IRL method focused on improving the apprentice policy should query. However, since the acquisition function of \citet{10.1007/978-3-642-04174-7_3} is formulated using entropy of possible actions probabilities, examples of this type could have $\alpha^{\text{Lopes}}(s_0) < \alpha^{\text{Lopes}}(s_1)$, resulting in an inefficient use of budget. For example, given inverse temperature $\beta=2$, we obtain values
\begin{equation}
    \alpha^{\text{Lopes}}(s_0) = 0.860 \,, \quad \alpha^{\text{Lopes}}(s_1) = 1.0 \,,
\end{equation}
so this acquisition function would query $s_1$, where the policy \textit{already knows which action is optimal}, rather than $s_0$ where there is key information to be gained. By contrast, assuming single-state queries, our PAC acquisition function would choose $s_0$, since in state the regret is already known to be $0$ for the current apprentice policy so there is no regret information to be gained there.

\paragraph{ActiveVaR \citep{brown2018}} This acquisition function for getting single-state annotations is equal to a particular quantile of the posterior distribution over regret of the apprentice policy starting from that state. In this example, we use a 0.9 quantile, though the example is robust with respect to the exact value. Consider an environment with two states labelled $s_{0,1}$ and two actions $a_{1,2}$ as shown in Figure \ref{fig:brown_env}. Both actions in state $s_0$ lead to $s_1$, one with reward $+2$ and the second with $-2$ (but we do not know which is which). In $s_1$, both actions lead to a terminal state, and give a reward of $-10$ and $+10$. Since the potential downside of any policy is maximal at $s_0$ ($-12$), the acquisition function would query $s_0$. On the other hand, querying state $s_1$ to distinguish the $\pm10$ rewards would yield a greater reduction in expected regret. 

\begin{figure}
    \centering
    \includegraphics[width=0.4\linewidth, trim=0 250 0 250, clip]{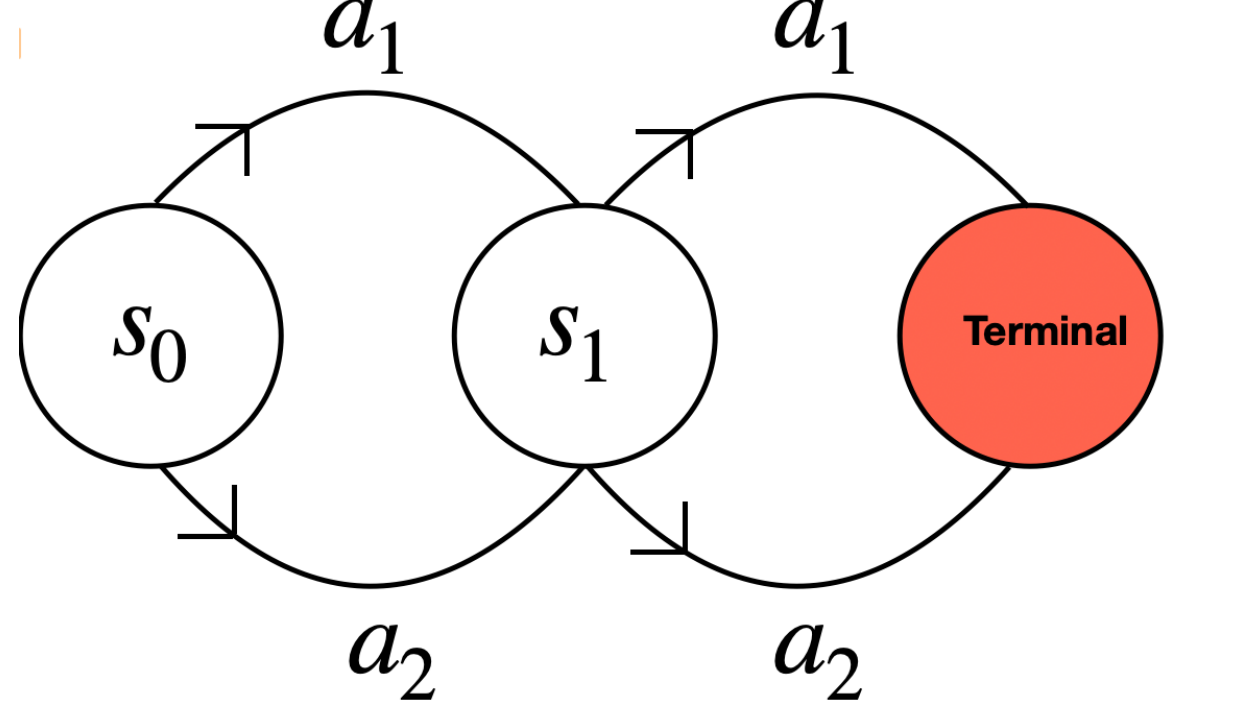}

    \caption{Two state environment to demonstrate a failure mode of \citet{brown2018}.}
    \label{fig:brown_env}
\end{figure}

To get even more concrete, consider an intermediate policy which \textit{knows} the absolute values of all the rewards, but not the relative signs: (i.e. $(+2,-2$) and $(-2,+2)$ are equally likely for $\left(r(s_0,a_1), r(s_0,a_2)\right)$, as are $(+10,-10)$ and $(-10, +10)$ for $\left(r(s_1, a_1), r(s_1, a_2)\right)$. We can easily compute
\begin{equation}
    \alpha^{\text{Brown}}(s_0) = 2 + \gamma 10 \,, \quad \alpha^{\text{Brown}}(s_1) = 10 \,,
\end{equation}
so for a sufficiently large discount factor, state $s_0$ would be queried by this acquisition function. 
We can compute the reduction in expected regret after querying each of these states. The initial expected total regret for any apprentice policy, averaged over a uniform initial state distribution is
\begin{equation*}
    \mathbb{E}_r[R_{\pia,r}] = \frac{1}{2} \big( V^*(s_0) - V^\pi(s_0) \big) + \frac{1}{2} \big(  V^*(s_1) - V^\pi(s_1) \big) = \frac{1}{2}( 2 + \gamma 10 - 0) + \frac{1}{2}(10 - 0)=6+10\gamma\,.
\end{equation*}
If we query the expert at $s_0$ and as a result switch the apprentice action at $s_0$ to the optimal one (which means we get a reward of $2$ in $s_0$ while still getting an expectation of $0$ in $s_1$, we get an expected regret of
\begin{equation*}
  \frac{1}{2} \big( (2 + \gamma10) -2 \big) + \frac{1}{2} \big(10 - 0 \big) = 5+5\gamma \,
\end{equation*}
while if we query at $s_1$ we get an expected regret of
\begin{equation*}
  \frac{1}{2} \big((2+\gamma 10) -\gamma10 \big) + \frac{1}{2} \big(10 - 10 \big) = 1 \,.
\end{equation*}
We therefore observe that whilst \citet{brown2018} would query $s_0$, querying $s_1$ yields a greater reduction in expected regret and also better tightening of the PAC condition (achieving it for any $\epsilon>1$ and any $\delta$). For a sufficiently high $\epsilon$, our PAC-EIG acquisition function would correctly query $s_1$ (though for a small $\epsilon$, it would recognize that both states need to be queried to satisfy the PAC criterion and would be indifferent between them).

While this is a possible conceptual objection to this acquisition function (which could be addressed by moving from regret to our immediate regret), we think this acquisition function is a strong option, which is also confirmed by empirical results.

\paragraph{Action Entropy \citep{kweon2023}} In its single-state-annotation version, the value of this acquisition function in a state is equal to the entropy of the posterior predictive distribution over expert actions in that state. Consider a situation where the learner has perfect knowledge of the action values in a particular state, but at least three actions in this state are equivalent (result in the same reward and next state distribution) and tied as optimal. Then the expert Boltzmann policy will assign uniform probabilities among these actions and due to high action entropy this state will always be queried in favour of any where one of two actions is best, but the learner does not know which of the two (assuming all other actions are known to be strongly suboptimal and thus unlikely to get queried). We offer an example of this in Figure \ref{fig:illustrative_pedagogical} which renders this acquisition function useless since it will only ever query the jail cell without gaining any information (and this is the case in both the single-state-annotation version, and a trajectory based version), while both our acquisition functions keep gathering useful information and improving the both the posterior reward entropy and the regret of the apprentice policy (including the PAC bounds).

For the single-state-annotation version, an even stronger counter example applies: if we allow querying terminal states in any environment that has them, the method always queries these, since actions have no effect, so a Boltzmann rational policy would be uniform and thus have maximum entropy.

\section{Theoretical Analysis}
\label{sup:theory}

We will establish an upper bound on the expected number of expert demonstrations needed to find a policy satisfying an $\epsilon$-$\delta$-PAC criterion using the PAC-EIG acquisition function. As outlined in Section~\ref{sec:method}, we will assume that during the learning process, the apprentice policy is one that in each state, maximizes the probability of taking an optimal action.
The proof strategy proceeds in three steps:
\begin{enumerate}
    \item First, we show that if a policy has a significant (overall) regret, there must exist a state where the policy's action is significantly suboptimal in terms of immediate regret. (Lemma~\ref{lemma:q-val-difference})
    
    \item Building on this, we prove that if a policy is not $(\epsilon,\delta)$-PAC, then there exists a state where the difference between optimal and apprentice policy's optimal Q-values is lower-bounded by $\epsilon(1-\gamma)$ with probability at least $\delta/|\sts|$.
    
    \item Finally, we show that in such cases, observing an expert demonstration from an appropriately chosen initial state provides a guaranteed minimum amount of information about whether the policy is approximately correct. Since wea function of can only gain a finite amount of information (bounded by the entropy of our prior), this leads to a bound on the number of demonstrations needed.
\end{enumerate}

We begin with our first lemma, which connects overall policy regret to statewise immediate regret, i.e. differences in optimal Q-values. This could be viewed as a corollary of the Performance Difference Lemma~\citet{kakade2002}, but we will prove it directly for completeness.

\begin{lemma}
\label{lemma:q-val-difference}
Let $\pi$ be any policy, $r$ any reward function, and
\[
R^\pi_r = \E_{s_0 \sim \rho_0} \left[ V^*_r(s_0) - V^\pi_r(s_0) \right] \geq 0,
\]
the regret of that policy. Then there exists a state $s \in \sts$ such that
\[
R^*_{\pi,r}(s):=Q^*_r(s, \pi^*_r(s)) - Q^*_r(s, \pi(s)) \geq (1 - \gamma) R^\pi_r.
\]
\end{lemma}

\begin{proof}
Let us define
\[
\Delta_Q = \max_{s \in \sts} \left[ Q^*_r(s, \pi^*_r(s)) - Q^*_r(s, \pi(s)) \right].
\]
We will prove the lemma by showing that $R^\pi_r \leq \Delta_Q / (1 - \gamma)$.

Since $Q^\pi_r(s, \pi(s)) \leq Q^*_r(s, \pi(s))$ (because $Q^*_r$ is the optimal Q-function), we have
\begin{align*}
V^*_r(s) - V^\pi_r(s) &= Q^*_r(s, \pi^*_r(s)) - Q^\pi_r(s, \pi(s)) \\
&\geq Q^*_r(s, \pi^*_r(s)) - Q^*_r(s, \pi(s)) \\
&\geq 0 \,.
\end{align*}
Using the Bellman equation, for any state $s\in\sts$ we can write 
\begin{align*}
     V^*_r(s) - V^\pi_r(s) &= Q^*_r(s, \pi^*_r(s)) - Q^\pi_r(s, \pi(s)) \\
     &= Q^*_r(s, \pi^*_r(s)) - Q^*_r(s, \pi(s)) + Q^*_r(s, \pi(s)) - Q^\pi_r(s, \pi(s)) \\
     &\leq \Delta_Q + \bigl(r(s, \pi(s)) + \gamma \E_{s'|s,\pi(s)} [V^*_r(s')]\bigr) - \bigl(r(s, \pi(s)) + \gamma \E_{s'|s,\pi(s)}V^\pi_r(s')]\bigr) \\
     &= \Delta_Q + \gamma \E_{s'|s,\pi(s)} [V^*_r(s')-V^\pi_r(s')] \\
    &\leq \Delta_Q + \gamma \max_{s'} [V^*_r(s')-V^\pi_r(s')].
\end{align*}
Here, the first equality just replaces state values by the corresponding Q-values, the second line adds and subtracts the same term, the third line uses the definition of $\Delta_Q$ for the first term and expands the latter two Q-values using the Bellman equation, the fourth just cancels out the repeated reward term. The final inequality follows because the expectation over next states is bounded by the maximum.

Since this inequality holds for all $s\in\sts$, it holds also for the state maximizing the left-hand side, so we get
\begin{equation}
     \max_{s} [V^*_r(s)-V^\pi_r(s)] \leq
     \Delta_Q + \gamma \max_{s'} [V^*_r(s')-V^\pi_r(s')]\;,
\end{equation}
which can be readily rearranged into
\[ \max_{s} [V^*_r(s)-V^\pi_r(s)]  \leq \frac{\Delta_Q}{1-\gamma}. \] 
Thus
\begin{align}
R^\pi_r &= \E_{s_0 \sim \rho_0} \left[ V^*_r(s_0) - V^\pi_r(s_0) \right] \leq \max_{s} [V^*_r(s)-V^\pi_r(s)] \\ 
& \leq \frac{\Delta_Q}{1-\gamma} = \frac{1}{1-\gamma} \max_{s \in \sts} \left[ Q^*_r(s, \pi^*_r(s)) - Q^*_r(s, \pi(s)) \right],
\end{align}
which completes the proof.
\end{proof}

\begin{lemma}
\label{lemma:probabilistic-q-val-diff}
Let $\pi$ be the apprentice policy at step $n$. For any $\delta \in (0,\frac{1}{2}]$, let $R_{n,\delta}^\pi$ be the $(1-\delta)$-quantile of the regret distribution with respect to the current posterior distribution over rewards, i.e., $R_{n,\delta}^\pi$ satisfies
\[
\prob_{r|\D_n} [R^\pi_r \geq R_{n,\delta}^\pi] = \delta.
\]

Then, there exists a state $s \in \sts$ such that
\[
\prob_{r|\D_n} \left[ Q^*_r(s, \pi^*_r(s)) - Q^*_r(s, \pi(s)) \geq (1 - \gamma) R_{n,\delta}^\pi \right] \geq \frac{\delta}{|\sts|}.
\]
\end{lemma}

\begin{proof}
Let us define the set of reward functions under which $\pi$ has high regret:
\[
\mathcal{H} = \left\{ r : R^\pi_r \geq R_{n,\delta}^\pi \right\}.
\]
By definition of the quantile $R_{n,\delta}^\pi$, we have
\[
\prob_{r|\D_n} [r \in \mathcal{H}] = \delta.
\]

For each $r \in \mathcal{H}$, applying Lemma~\ref{lemma:q-val-difference}, we know there exists a state $s_r \in \sts$ such that
\[
Q^*_r(s_r, \pi^*_r(s_r)) - Q^*_r(s_r, \pi(s_r)) \geq (1 - \gamma) R^\pi_r \geq (1 - \gamma) R_{n,\delta}^\pi.
\]
Let $\sts_r$ be the set of such states for a reward function $r \in \mathcal{H}$ and $\sts_\mathcal{H}$ the collection of such states across all rewards $r \in \mathcal{H}$. Since the state space $\sts$ is finite with cardinality $|\sts|$, by the pigeonhole principle, there must exist at least one state $s \in \sts$ such that
\[
\prob_{r|\D_n} \left[ r \in \mathcal{H} \text{ and } s \in \sts_r \right] \geq \frac{\delta}{|\sts|}.
\]

For this state $s$, whenever $r \in \mathcal{H}$ and $s\in\sts_r$, we have
\[
Q^*_r(s, \pi^*_r(s)) - Q^*_r(s, \pi(s)) \geq (1 - \gamma) R_{n,\delta}^\pi.
\]

Therefore,
\[
\prob_{r|\D_n} \left[ Q^*_r(s, \pi^*_r(s)) - Q^*_r(s, \pi(s)) \geq (1 - \gamma) R_{n,\delta}^\pi \right] \geq \frac{\delta}{|\sts|},
\]
which completes the proof.
\end{proof}

This lemma extends our previous result to the probabilistic setting of Bayesian IRL. While Lemma~\ref{lemma:q-val-difference} showed that high regret implies the existence of a state with poor action choice, this lemma shows that if our policy has a significant probability of high regret, there must be at least one state where it has a significant probability of making a poor action choice.

Now we will build on this lemma to formulate our first theorem to take a step further: if we apply this lemma to a policy that has a significant probability of being approximately correct in each state, but the lemma also gives us a state where it has a significant probability of making a poor choice, we have two contradictory hypotheses that an expert demonstration can help us resolve -- we formalize this as a lower bound on the information gain from observing the expert action in such a state.

\thmmineig*


\begin{proof}
Let $\pia$ be an apprentice policy that in each state maximizes the probability of taking an optimal action, i.e., $ \pia(s)\in\argmax_a \prob[Q^*(s,\pi^*(s)) - Q^*(s,a)=0|\D_n],\;\forall s\in\sts$. To informally outline our proof strategy: we will prove the theorem by showing that under its assumptions, there exists a state $s$ and an alternative action $a'\neq\pia(s)$ that has a chance of being significantly better than the apprentice action. If $a'$ is significantly better, it is significantly more likely to get selected by the expert than if the apprentice action is optimal. Since the observation distributions in the two cases are different, this allows us to put a lower bound on the expected information gained by observing the expert at $s$. Now let us turn to the proof in full detail.

Under the assumptions of this theorem and using Lemma~\ref{lemma:probabilistic-q-val-diff}, there exists a state $s$ such that
\[
\prob_{r|\D_n} \left[ Q^*_r(s, \pi^*_r(s)) - Q^*_r(s, \pia(s)) \geq (1 - \gamma) \epsilon \right] \geq \frac{\delta}{|\sts|}.
\]
Going forward, let us fix $s$ to denote one such state, and let us denote by $e_*$ the event on the left, $\left\{Q^*_r(s, \pi^*_r(s)) - Q^*_r(s, \pia(s)) \geq (1 - \gamma) \epsilon\right\}$. 
Conditioned on $e_*$, one of the $|\as|-1$ alternative actions must be optimal so by the pigeonhole principle, there must exist an action that, conditional on $e_*$ has probability of at least $\frac{1}{|\as|-1}$ of being optimal. Going forward, let us fix $a'$ to be such an action. This implies that
\begin{equation}
    \label{eq:en-lb}
    \prob\!\left[\,\en\right]:=\prob\!\left[\,Q^*_r(s, a') - Q^*_r(s, \pia(s))\geq (1 - \gamma) \epsilon \right]\geq\frac{\delta}{|\sts|(|\as|-1)}.
\end{equation}
($\en$ standing for the apprentice action \emph{not} being approximately correct relative to $a'$) since the action $a'$ being optimal is a subset of the event $\en$ conditioned on $e_*$. 

On the other hand, since $\pia$ was chosen as the policy maximizing the probability of being "correct", we must have
\begin{equation*}
  \prob\left[  Q^*_r(s, \pi^*_r(s)) - Q^*_r(s, \pia(s)) = 0 \right] \geq \frac{1}{|\as|},
\end{equation*}
since one of the $\as$ actions must be optimal in each state, so by the pigeonhole principle, there must exist at least one action that has a probability of at least $\frac{1}{|\as|}$ of being optimal.

This implies that with probability at least $1/|\as|$, $\pia(s)$ is an optimal action, and for any other action $a'$, $Q^*_r(s, a') - Q^*_r(s, \pia(s)) \leq 0$. Let us denote the event that $\pia(s)$ is optimal, or "correct", by $\ec$, so
\begin{equation}
  \label{eq:ec-lb}
  \prob\left[ \ec \right] \geq \frac{1}{|\as|}.
\end{equation}
For completeness, we can denote by $\ea$ the "approximately correct" complement of the events $\ec$ and $\en$.

Now, if we denote by $A$ the action that the expert would take in state $s$ seen as a random variable, we can decompose the mutual information between $A$ and the ternary variable $E^\pia_{s,a'}$ that can take the values $\ec$, $\ea$, or $\en$ as
\begin{align*}
I(A;E^\pia_{s,a'}) &= \sum_{e\in\{\ec,\ea,\en\}}\prob[e] D_{\text{KL}}(p(A|e)\| p(A)) \\
&\geq \prob[\ec] D_{\text{KL}}(p(A|\ec)\| p(A)) + \prob[\en] D_{\text{KL}}(p(A|\en)\| p(A))
\end{align*}

To finish the proof, we need to put a lower bound on this mutual information. We have already given a lower bound on $\prob[\en]$ (Eq.~\ref{eq:en-lb}) and  $\prob[\ec]$ (Eq.~\ref{eq:ec-lb}). We will now provide a lower bound on the sum of the corresponding KL terms by first lower-bounding the total variation distance between $p(A|\en)$ and $p(A|\ec)$.

In the event $\ec$, we have $Q^*_r(s,\pia(s)) \geq Q^*_r(s, a')$, so under the Boltzmann-rational policy, for any given reward $r_c$ compatible with $\ec$, we have
\begin{align*}
    \prob\left[A=\pia(s)|r_c\right] \geq \prob\left[A=a'|r_c\right].
\end{align*}

On the other hand, in the case of $\en$ and a reward $r_{\text{n}}$ compatible with it, we have $Q^*_r(s, a') \geq Q^*_r(s, \pia(s)) + (1-\gamma)\epsilon$. Thus, under the Boltzmann-rational expert policy, we have
\begin{equation*}
    \prob\left[A=a'|r_{\text{n}}\right] \geq e^{\beta(1-\gamma)\epsilon} \prob\left[A=\pia(s)|r_{\text{n}}\right]\;,
\end{equation*}
so 
\begin{equation*}
    \prob\left[A=a'|r_{\text{n}}\right] -  \prob\left[A=\pia(s)|r_{\text{n}}\right] \geq (1-e^{-\beta(1-\gamma)\epsilon}) \prob\left[A=a'|r_{\text{n}}\right].
\end{equation*}
Marginalizing over the reward, we get
\begin{align}
    \label{eq:en-prob-diff}
    \prob\!\left[A=a'|\en\right] - \prob\!\left[A=\pia(s)|\en\right] &\geq (1-e^{-\beta(1-\gamma)\epsilon}) \prob\left[A=a'|\en\right] .
\end{align}
If $a'$ is optimal, its probability of being selected by the expert is at least $1/|\as|$. In fact, conditional on $\en$, its probability must be \emph{strictly} greater than that of $\pia(s)$, so its probability of being selected if optimal is strictly greater than $\frac{1}{|\as|}$. Since $a'$ was chosen so that it has a probability of at least $\frac{1}{|\as|-1}$ of being optimal conditional on $e_*$, it must also have a probability of at least $\frac{1}{|\as|-1}$ conditional on $\en$ (since $\en\subseteq e_*$ and $\{a\;\text{optimal}\}\cap e_* \subseteq e_n$). Thus, we have
\begin{align*}
    \prob\left[A=a'|\en\right] &=  \prob\left[A=a'|a'\;\text{optimal and }\en\right] \prob\left[a'\;\text{optimal }|\en\right] \\&\;\;\;\;+  \prob\left[A=a'|a' \text{ not optimal and }\en\right]\prob\left[a'\;\text{ not optimal }|\en\right] \\
    &\geq \prob\left[A=a'|a'\;\text{optimal and }\en\right] \prob\left[a'\;\text{optimal }|\en\right] \\
    &> \frac{1}{|\as|(|\as|-1)}.
\end{align*}    

Combining this with Equation~\ref{eq:en-prob-diff} gives us
\begin{align*}
    \prob\!\left[A=a'|\en\right] - \prob\!\left[A=\pia(s)|\en\right] &> \frac{1-e^{-\beta(1-\gamma)\epsilon}}{|\as|(|\as|-1)}   \;.
\end{align*}

We can use this to put a lower bound on the total variation distance between the posterior predictive distributions under the events $\en$ and $\ec$:
\begin{align*}
D_{\text{TV}}&\left(p(A|\ec),p(A|\en)\right) \\
    &\geq\frac{1}{2} \biggl( \Bigl| \prob\left[A=\pia(s)|\ec\right] - \prob\left[A=\pia(s)|\en\right]\Bigr| + \Bigl|\prob\left[A=a'|\en\right] - \prob\left[A=a'|\ec\right]\Bigr| \biggr) \\
    &\geq \frac{1}{2} \biggl(\Bigl(\prob\left[A=a'|\en\right] - \prob\left[A=\pia(s)|\en\right]\Bigr) - \Bigl(\prob\left[A=a'|\ec\right]-\prob\left[A=\pia(s)|\ec\right]\Bigr)\biggr) \\
    &> \frac{1}{2} \frac{1-e^{-\beta(1-\gamma)\epsilon}}{|\as|(|\as|-1)}
\end{align*}
where we used the definition of total variation distance in the first step (omitting non-negative terms corresponding to actions other than $\pia(s)$ and $a'$), the triangle inequality, dropping absolute value, and rearranging terms in the second step, and, in the last step, using the lower bound we just derived and dropping the second term, since $\prob\left[A=a'|\ec\right]-\prob\left[A=\pia(s)|\ec\right]\leq 0$.

Applying Pinsker's inequality gives us
\begin{align*}
D_{\text{KL}}(p(A|\ec)\|p(A)) + D_{\text{KL}}(p(A|\en)\|p(A)) &\geq 
2(D_{\text{TV}}(p(A|\ec), p(A))^2 + D_{\text{TV}}(p(A|\en),p(A))^2) \\
    &\geq
    (D_{\text{TV}}(p(A|\ec), p(A)) + D_{\text{TV}}(p(A|\en),p(A)))^2 \\
    &\geq D_{\text{TV}}(p(A|\en),p(A|\ec)))^2 \\
    &> \frac{1}{4|\as|^2(|\as|-1)^2} (1- e^{-\beta(1-\gamma)\epsilon})^2.    
\end{align*}
where we applied the inequality $(a+b)^2\leq 2(a^2+b^2)$ in the second step, and the triangle inequality in the third.

This finally allows us to establish that 
\begin{align*}
    I(A;E^\pia_{s,a'}) &= \sum_{E^\pia_{s,a'}\in\{e_{\text{n}},\ec,e_{\text{a}}\}} \prob[E^\pia_{s,a'}] D_{\text{KL}}(p(A|E^\pia_{s,a'})\| p(A)) \\
    &\geq \min\{ \prob[\ec],\prob[e_{\text{n}}] \} \left( D_{\text{KL}}(p(A|\ec)\|p(A)) + D_{\text{KL}}(p(A|e_{\text{n}})\|p(A)) \right) \\
    &> \min\left\{ \frac{1}{|\as|}, \frac{\delta}{(|\as| - 1)|\sts|} \right\}   \frac{(1-e^{-\beta(1-\gamma)\epsilon})^2}{4|\as|^2(|\as|-1)^2}  \\
    &=  \frac{\delta(1-e^{-\beta(1-\gamma)\epsilon})^2}{4|\as|^2(|\as|-1)^3|\sts|}.
\end{align*}
The first inequality follows from dropping the non-negative term corresponding to the event $e_{\text{a}}$ and taking the minimum of the remaining two probability terms. The second inequality just plugs in results previously derived in this proof. The final step resolves the minimum as its second term using the assumption that $\delta\leq\frac{1}{2}$ and $|\as|\geq 2$.

Since the random variable $E^\pia_{s,a'}$ is coarser than the variable $E^\pia$, which is the collection of the variables $E^\pia_{s,a}$ across all state-action pairs, we have $I(A;E^\pia) \geq I(A;E^\pia_{s,a'})$, which completes the proof.
\end{proof}

Now that we have a lower bound on the information that we gain in each step, we can use it to bound the number of steps needed to reach the PAC condition.

\thmnumsteps*

\begin{proof}
    The minimal expected information gain guaranteed by Theorem~\ref{thm:min-eig} is fully derived from components of the random variable $E$. Thus, at every step of active learning where we have not yet achieved the PAC criterion, we can gain at least ${\text{EIG}}_{\text{min}} (\epsilon, \delta)$ information about $E$.
    
    Let $H_n$ denote the expected entropy of $E$ after $n$ steps of active learning; in particular, $H_0$ is the entropy of the prior distribution over $E$. By the properties of entropy and information gain:
    \begin{enumerate}
        \item $H_n \geq 0$ for all $n$ (non-negativity of entropy)
        \item $H_{n+1} \leq H_n - {\text{EIG}}_{\text{min}} (\epsilon, \delta)$ for all $n$ where the PAC criterion is not met (guaranteed information gain)
    \end{enumerate}
    
    Let $N$ be the number of steps needed to reach the PAC criterion. Then:
    \begin{align*}
        0 &\leq H_N \\
        &\leq H_0 - N \cdot {\text{EIG}}_{\text{min}} (\epsilon, \delta)
    \end{align*}
    
    Solving for $N$:
    \begin{equation}
        N \leq \frac{H_0}{{\text{EIG}}_{\text{min}} (\epsilon, \delta)}
    \end{equation}
    
   The result follows by substituting the expression for ${\text{EIG}}_{\text{min}} (\epsilon, \delta)$ from Theorem~\ref{thm:min-eig} and the fact that $H_0 \leq h_{\text{max}}$ since $h_{\text{max}}$ is an upper bound on the prior entropy of $E$.
\end{proof}

\begin{corollary}
For any prior distribution over rewards, the expected number of steps to reach the PAC condition is at most
\begin{equation}
    \log(3)|\sts||\as|^2 / {\text{EIG}}_{\text{min}}(\epsilon, \delta) = \frac{4\log(3)|\as|^3(|\as|-1)^4|\sts|^2}{\delta (1-e^{-\beta(1-\gamma)\epsilon})^2}.
\end{equation}
\end{corollary}

\begin{proof}
The random variable $E=\{E^\pia\;\forall \pia\}$ aggregates $|\sts||\as|(|\as|-1)$ ternary random variables (in each state, apprentice policies can take $|\as|$ actions, and we are considering each action's immediate regret relative to $|\as|-1$ alternative actions), so $E$ can take at most $3^{|\sts||\as|(|\as|-1)}$ values. Thus, its maximum entropy is $-\log(1/3^{|\sts||\as|(|\as|-1)}) = |\sts||\as|(|\as|-1)\log(3)$. The result follows by plugging this maximum entropy into Theorem~\ref{thm:num-steps}.
\end{proof}

While we do not claim this bound is tight, it provides a useful characterization of how the sample complexity scales with the problem parameters. In particular, it shows polynomial dependence on the size of the state and action spaces, and inverse dependence on both the allowed suboptimality $\epsilon$ and failure probability $\delta$. More importantly, on a qualitative level, the result shows that the learning process does continue as long as the PAC condition is not satisfied, so it does not get stuck forever querying an uninformative state as is the case with e.g. the acquisition function of \citet{kweon2023}.

\subsection{Notes on possible improvements}

\subsubsection{Tighter bound for large state spaces and limited horizon}
Note that the bound from Lemma~\ref{lemma:probabilistic-q-val-diff} can be tightened if the state space is large and only a subset is reachable within an effective horizon. In that case $|S|$ can be replaced by the number of states reachable from the initial states within $1/(1-\gamma)$ steps. Also, if there is a limited time horizon, or the apprentice policy always reaches a terminal state in a certain number of steps, we can use this to tighten the $1/(1-\gamma)$ effective horizon and the associated $(1-\gamma)$ factor in the exponent of our theoretical results.

\subsubsection{Dependence on the number of actions}
The bound also considers reducing the entropy for every possible apprentice policy. If aiming for a tight PAC bound, the training will generally be able to quickly reduce the space of plausible apprentice policies and then just keep refining the bound for that policy in most states reducing the dependence on the number of actions in practice.
\section{Regret-based Acquisition Functions: Journey toward PAC-EIG}
\label{sup:alternative-acquisition-functions}

In this appendix, we discuss alternative formulations of acquisition functions for regret-focused active IRL, which could help to see the reasoning process that led to PAC-EIG, and explain why some seemingly obvious alternatives were not chosen.

\subsection{From Apprentice Return to Regret}

When the goal is to produce a well-performing apprentice policy (as opposed to learning the reward for its own sake), a natural starting point is to directly optimize the apprentice's expected return. This suggests minimizing the loss:
\begin{equation}
    \mathcal{L}_{\text{ret}}(\xi_1,\dots,\xi_N) = -\mathbb{E}_{s_0\sim\rho} \mathbb{E}_{\tau|s_0,\pi^{\text{A}}_N} G_r(\tau) 
\end{equation}
where $\pi^{\text{A}}_N$ is the apprentice policy after observing $N$ expert trajectories from initial states $\xi_1,\dots,\xi_N$.

Since the optimal value $V^*(s_0)$ is independent of our choice of queries, minimizing $\mathcal{L}_{\text{ret}}$ is equivalent to minimizing the regret loss:
\begin{equation}    
    \mathcal{L}_{\text{reg}}(\xi_1,\dots,\xi_N) = R_r^{\pi^{\text{A}}_N} = \mathbb{E}_{s_0\sim\rho} \left[ V^*_r(s_0) - V^{\pi^{\text{A}}_N}_r(s_0) \right]
\end{equation}

\subsection{The Challenge of Direct Regret Optimization}

Directly optimizing this regret loss is computationally intractable even in the greedy case. The one-step acquisition function would be:
\begin{equation}
    \alpha^{\text{reg}}_n(\xi) = -\mathbb{E}_{r|\mathcal{D}_n} \mathbb{E}_{\tau_{n+1}|\xi,\pi^{\text{E}}_r} R^{\pi_{n+1}^{\text{A}}}_r
\end{equation}

Computing this requires:
1. For each possible reward $r$ in our posterior
2. For each possible expert trajectory $\tau$ from initial state $\xi$
3. Computing the updated posterior $p(r|\mathcal{D}_n \cup \{\tau\})$
4. Finding the optimal apprentice policy for this updated posterior
5. Evaluating its regret

This nested optimization involving repeated Bayesian IRL updates is prohibitively expensive.

\subsection{Information Gain About Regret}

Following the approach in Bayesian optimization \citep{wang2017}, rather than directly optimizing the hard-to-compute expected improvement, we can instead maximize information gain about the quantity of interest. This suggests the acquisition function:
\begin{equation}
    \alpha^{\text{Regret-EIG}}_n(s_0) = I(\tau; R_r^{\pi^A} | s_0, \mathcal{D}_n)
\end{equation}

However, this formulation has a critical flaw. Consider this example:
\begin{itemize}
    \item In state $s$, the apprentice can take action $a_0$ yielding return 0
    \item Actions $a_1$ and $a_2$ yield returns of +50 and -100, but we don't know which is which
    \item With equal probability on both orderings, the apprentice chooses $a_0$ (expected return 0 vs -25)
    \item The regret is known with certainty to be 50
    \item Since there's no uncertainty about regret, Regret-EIG assigns zero value to querying this state
    \item Yet the apprentice is definitely choosing suboptimally!
\end{itemize}

\subsection{Immediate Regret EIG}

The solution is to decompose regret more finely. The total regret can be written as:
\begin{equation}
    R^{\pi^A}_r = \mathbb{E}_{\tau\sim\rho,\pi^A} \sum_{s_t,a_t\in\tau} \gamma^t \underbrace{\left[V^*_r(s_t) - Q^*_r(s_t,\pi^A(s_t))\right]}_{R^*_{\pi^A,r}(s_t)}
\end{equation}

where $R^*_{\pi^A,r}(s)$ is the \emph{immediate regret} -- the value lost by following the apprentice policy in state $s$ without considering future consequences.

This can be further decomposed per action:
\begin{equation}
    R^*_{\pi^A,r}(s) = \max_a R^*_{\pi^A,r}(s,a)
\end{equation}
where $R^*_{\pi^A,r}(s,a) = \max\{0, Q^*_r(s,a) - Q^*_r(s,\pi^A(s))\}$.

The Immediate Regret EIG acquisition function is then:
\begin{equation}
    \alpha^{\text{IR-EIG}}_n(s_0) = I(\tau; R^* | s_0, \mathcal{D}_n)
\end{equation}
where $R^* = (R^*_{\pi^A,r}(s,a))_{s\in\mathcal{S}, a\in\mathcal{A}}$.

This formulation correctly identifies informative states in our earlier example, as there is high uncertainty about which action has higher immediate regret.

\subsection{Discretization and Connection to PAC-EIG}

For practical computation, the continuous immediate regret values must be discretized. Different discretization schemes lead to different acquisition functions:

\begin{enumerate}
    \item \textbf{Multi-bucket discretization}: Using buckets of [0, $\epsilon/2$], [$\epsilon/2$, $\epsilon$], [$\epsilon$, $\infty$) like in PAC-EIG can be taken further to allow for a finer approximation of the IR-EIG. This can be viable for single state-queries, but growing the number of categories becomes untenable once we start considering the full trajectory demonstration. 
    
    \item \textbf{Two-bucket PAC discretization}: Using just two buckets -- acceptable regret $[0, \epsilon]$ and unacceptable regret $(\epsilon, \infty)$ -- directly captures what matters for PAC guarantees. However, the theoretical arguments as shown above do not directly apply to this case, since we loose the middle bucket that ensures separation between the two sufficiently different expert action distributions. 
\end{enumerate}

Our three-bucket PAC discretization is not only computationally more tractable but also theoretically motivated: it focuses information gathering on exactly what we need to know to provide formal reliability guarantees.

\subsection{Summary}

The progression from expected return optimization to PAC-EIG illustrates how principled information-theoretic thinking, combined with practical computational constraints and theoretical objectives, leads to an effective acquisition function. While IR-EIG with fine discretization might provide marginally more information in some cases, PAC-EIG strikes the optimal balance between theoretical guarantees, computational efficiency, and practical effectiveness. 
\section{Experiment details}
\label{sup:exp-details}

\subsection{Basic parameter values}
In the three environments (structured 6x6, random 8x8, and random 10x10) we used $\beta=4,2,4$ respectively, $\gamma=0.9$, and an infinite horizon (but all environments contained terminal states). We started with an empty set of demonstrations (implemented as a single, uninformative observation of a dummy sink state) and then ran active learning for 150 steps.

For ActiveVaR, we used $\delta=0.05$ (same as the original paper). For policy entropy, we used the entropy of the discretized distribution for each action (as proposed by the authors) with K=10 buckets. 
For our PAC-EIG acquisition function we used $(1-\gamma)\epsilon=0.01$ for the PAC condition.

\subsection{Environments}
The gridworld environments have 5 actions, corresponding to staying in place and moving in the four directions. Furthermore, there is a probability of 0.1 of random action being executed instead of the intended one. If an action would result in crossing the edge, the agent instead remains in place. The gridworlds use a state-only reward (awarded upon executing any action in the given state).

The 8x8, and 10x10 fully random environments were generated as follows:
\begin{enumerate}
    \item Each state was assigned a random reward drawn independently from $\mathcal{N}(0,3)$ (i.e. mostly yielding rewards between -10 and 10). 
    \item Each state was then marked as terminal with an independent probability of 0.1.
    \item The top $10\%$ of states with highest reward were further marked as terminal (producing terminal goal states, which may, however, sometimes be avoided by the optimal policy in favour of staying forever in other positive states).
    \item The initial state distribution is either uniform across the whole state space, or, in the case of the 10x10 gridworld, 2 non-terminal initial states were chosen randomly uniformly.\footnote{Note that the implementation allows the two initial states to collide, producing only a single initial state in $1/81$ of the cases, but this was not the case for any of our 16 random seeds.}
\end{enumerate}

\subsection{Bayesian IRL methods}
Our active learning uses a Bayesian IRL method as a key component. In our experiments, we used two methods based on Markov chain Monte Carlo (MCMC) sampling: on the structured environment, we used PolicyWalk~\citep{ramachandran2007}, while on the environment with a different random reward in every state, we used the faster ValueWalk~\citep{bajgar2024}, which performs the sampling primarily in the space of Q-functions before converting into rewards. We also tried a method based on variational inference \citep{chan2021}, but we found its uncertainty estimates unreliable for the purposes of active learning.

For MCMC sampling, we used Hamiltonian Monte Carlo~\citep{duane1987} with the no-U-turns (NUTS) sampler~\citep{hoffman2014} and automatic step size selection during warm-up (starting with a step size of 0.1). At every step of active learning, we ran the MCMC sampling from scratch using all demonstrations available up to that point. We ran for 100 warm-up steps and then 200, 500, and 1000 on the three environments respectively. For subsequent usage, we use every other sample to reduce autocorrelation.

\subsection{Metrics}
On the first two environments, we use KNN entropy estimation to calculate posterior entropy with K=5. This method is known to struggle in high dimensions, which we also observed in the case of the 10x10 gridworld (which has a 100-dimensional reward space), so there, we estimate the entropy by the entropy of a multivariate normal distribution with the mean and covariance matrix estimated from the MCMC samples. 

Regret was calculated relative to the expected return of the optimal policy, calculated using value iteration with a tolerance of 1e-5. Posterior regret samples were similarly calculated relative to the optimal return with respect to each of the posterior reward samples (which were calculated using the optimal Q-value samples which get produced by the Bayesian IRL methods). 

When aggregating true regret across environment instances, we also normalized the regret for each random environment instance by the average regret across all methods across the first 32 steps of active learning to account for the possibly different scales and different learning difficulties of the random environments.

\subsection{Implementation}
The experiments were implemented using Python 3.10, PyTorch 2.5.1, and Pyro 1.8.6. We will publish our full code for both the experiments and the associated result analysis on Github.

\subsection{Timing}
The computational time per step of active IRL is dominated by the time necessary to collect the Bayesian IRL MCMC samples, which ranges between 5 seconds for the 100+200 samples on the structured gridworld to about 5 minutes for the 100+1000 samples on the 10x10 gridworld in a single CPU thread. The overhead of all acquisition functions on top of that is below $0.03$ and can thus be considered negligible.

Reproducing all our experiments thus takes less than a day on a CPU with 128 threads (we used AMD Ryzen Threadripper 3990X at 2.2GHz).

\section{Notation overview}
\label{sup:notation}
\begin{table}[h]
\centering
\caption{Summary of notation used throughout the paper. If reward is omitted from a symbol otherwise depending on it, it means it is taken with respect to the true reward.}
\setlength{\tabcolsep}{6pt}
\begin{tabular}{ll}
\toprule
Symbol & Meaning \\\midrule
$\mathcal S$ & State spstate ace of the MDP \\
$\mathcal A$ & Action space \\
$P(s'\!\mid\!s,a)$ & Transition kernel \\
$r:\mathcal S\times\mathcal A\!\to\!\mathbb R$ & Expected reward function \\
$\gamma\!\in\!(0,1)$ & Discount factor \\
$t_{\max}$ & Maximum horizon (may be $\infty$) \\
$\rho$ & Initial-state distribution \\
$\pi^{E}$ & Expert policy (Boltzmann-rational with coefficient $\beta$) \\
$\pi^{E}_r$ & Hypothetical expert policy that would correspond to a reward $r$ \\
$\beta$ & Boltzmann rationality coefficient \\
$\mathcal D_n$ & Demonstration data after $n$ queries \\
$\pia,\;\pi^{\text{A}}_n$ & Apprentice policy (after $n$ queries) \\
$\pi^\star_r$ & Optimal policy for reward $r$ \\
$\tau=(s_0,a_0,\dots,s_T)$ & Trajectory \\
$\xi$ & Query (initial state for the next expert demonstration) \\
$V_r^\pi(s)$, $Q_r^\pi(s,a)$ & State- and action-value functions for policy $\pi$ and reward $r$ \\
$G_r(\tau)$ & Discounted return of trajectory $\tau$ under $r$ \\
$G_r^\pi$ & Expected discounted return of $\pi$ under $r$, $P$, and $\rho$ \\
$R^\pi_r(s_0)$ & Regret of $\pi$ from state $s_0$ \\
$R^\star_{\pi,r}(s)$, $R^\star_{\pi,r}(s,a)$ & Immediate regret of $\pi$ for state $s$ and state–action pair $s,a$ \\
$I(\tau; r\mid s_0,\mathcal D_n)$ & Mutual information between a trajectory and reward \\
$\text{EIG}$ & Expected information gain \\
$\alpha^{\text{RewEIG}}_n$ & Reward-EIG acquisition function \\
$\text{VaR}_\delta$ & $\delta$-value-at-risk of a loss random variable \\
$EIG_{\min}(\epsilon,\delta)$ & Per-step information-gain lower bound (Thm 3) \\
$h_{\max}$ & Upper bound on the prior entropy of $E$ \\
$\epsilon,\;\delta$ & PAC accuracy / confidence parameters \\
$|\mathcal S|,\;|\mathcal A|$ & Cardinalities of state and action spaces \\
\bottomrule
\end{tabular}
\end{table}

\end{document}